\let\clineorig\cline

\documentclass[default]{sn-jnl}

\usepackage{graphicx}%
\usepackage{multirow}%
\usepackage{amsmath,amssymb,amsfonts}%
\usepackage{amsthm}%
\usepackage{mathrsfs}%
\usepackage[title]{appendix}%
\usepackage{xcolor}%
\usepackage{textcomp}%
\usepackage{manyfoot}%
\usepackage{booktabs}%
\usepackage{algorithm}%
\usepackage{algorithmicx}%
\usepackage{algpseudocode}%
\usepackage{listings}%
\usepackage{natbib}%

\theoremstyle{thmstyleone}%
\newtheorem{theorem}{Theorem}
\theoremstyle{thmstyletwo}%
\newtheorem{example}{Example}%

\theoremstyle{thmstylethree}%

\input{Definitions}
\usepackage{outlines}
\usepackage{amsmath,amssymb,mathtools}
\usepackage{url}
\usepackage{booktabs}
\usepackage{algorithm}
\usepackage{algpseudocodex}
\usepackage{mathtools}
\usepackage{lineno}
\usepackage{multirow} 
\usepackage{tabularx}
\usepackage{tikz}
\usetikzlibrary{arrows.meta, positioning}
\usepackage{rotating}
\usepackage{enumitem} 
\usepackage{hyperref}
\hypersetup{colorlinks=false, pdfborder={0 0 0}}

\newtheorem{thm}{Theorem}
\newtheorem{prop}{Proposition}

\newcommand{\Perp}{\mathop{\!\perp\!\!\!\perp\!} }

\newcommand{\ctext}[1]{\raise0.2ex\hbox{\textcircled{\scriptsize{#1}}}}
\newcommand{\deffunc}{\mathbb{I}}

\begin{document}


\title[Detection of Unobserved Common Causes based on NML Code in Discrete, Mixed, and Continuous Variables]{Detection of Unobserved Common Causes based on NML Code in Discrete, Mixed, and Continuous Variables}


\author*[1]{\fnm{Masatoshi} \sur{Kobayashi}}\email{kobayashi-masatoshi453@g.ecc.u-tokyo.ac.jp}

\author[2]{\fnm{Kohei} \sur{Miyaguchi}}\email{miyaguchi@ibm.com}

\author[1]{\fnm{Shin} \sur{Matsushima}}\email{https://orcid.org/0000-0002-8160-4310}

\affil[1]{
    \orgdiv{Graduate School of Information Science and Technology}, 
    \orgname{The University of Tokyo}, 
    \orgaddress{\street{7-3-1 Hongo, Bunkyo-ku}, \city{Tokyo}, \postcode{113-8656}, \country{Japan}}
}

\affil[2]{
    \orgdiv{IBM Research -- Tokyo}, 
    \orgaddress{\city{Tokyo}, \country{Japan}}
}


\abstract{Causal discovery in the presence of unobserved common causes from observational data only is a crucial but challenging problem. 
We categorize all possible causal relationships between two random variables into the following four categories and aim to identify one from observed data: two cases in which either of the direct causality exists,
a case that variables are independent, and a case that variables are confounded by latent confounders. 
Although existing methods have been proposed to tackle this problem, they require unobserved variables to satisfy assumptions on the form of their equation models. 
In our previous study~\cite{kobayashi2022detection}, the first causal discovery method without such assumptions is proposed for discrete data and named  \textsf{CLOUD}. 
Using Normalized Maximum Likelihood (NML) Code, \textsf{CLOUD} selects a model that yields the minimum codelength of the observed data from a set of model candidates.
This paper extends \textsf{CLOUD} to apply for various data types across discrete, mixed, and continuous. We not only performed theoretical analysis to show the consistency of \textsf{CLOUD} in terms of the model selection, but also demonstrated that \textsf{CLOUD} is more effective than existing methods in inferring causal relationships by extensive experiments on both synthetic and real-world data. }




\maketitle

\section{Introduction}\label{sec1}
Intelligent systems that utilize accurate prediction based on data has enjoyed remarkable success by the development of machine learning methodology. 
It is expected that an accurate predictor learned from data possesses a certain form of information on the data. 
This expectation motivates us to render information extracted from data by a learning algorithm into a form in which humans can understand.
For this sake, it is considered that the study of causal inference, aiming at extracting the underlying causal mechanism, has gained prominence in the context of machine learning as well as other various fields \cite{peters2017elements, scholkopf2022causality}.

As it is impractical to perform randomized control trials in many cases, many studies have focus on inferring causal structures based soly on observational data \cite{pearl2009causality}. In the context of classical causal discovery, the decision whether $X$ is a direct cause of $Y$ or $Y$ is a direct cause of $X$ is already a hard problem so that standard methods assume there is no unknown common cause (causal sufficiency).
However, in practice, this assumption is often violated, which can results in the methods producing unreliable results. Therefore, causal discovery that allows for the presence of unobserved common causes becomes crucial.


We consider a new problem setting by revisiting Reichenbach's common cause principle quoted as follows~\cite[Principle 1.1]{peters2017elements}:

\begin{quote}
    If two random variables $X$ and $Y$ are statistically dependent,
    then there exists a third variable $C$ that causally influences both.
    (As a special case, $C$ may coincide with either $X$ or $Y$.)
    Furthermore, this variable $C$ screens $X$ and $Y$ from each other in the sense that given $C$, they become independent.   
\end{quote}
As a logical conclusion of the statement above, we can always categorize the relationship between $X$ and $Y$ into 4 cases:
1) $X$ causally influences ($X$ and) $Y$, 2) $Y$ causally influences $X$ (and $Y$), 3) there exists a third variable $C$ that causally influences $X$ and $Y$, and 4) $X$ and $Y$ are statistically independent.
We call a problem to decide which among those cases from data \emph{Reichenbach problem} and deal with it.
The solution for this problem is advantageous over the traditional methods in the sense that it does not require the prior knowledge on the nonexistence of unobserved confounder and then it is widely applicable.

There has been many existing studies on the detection of unobserved common cause as well. As we discuss in later sections, those methods require a type of assumptions such that unobserved confounding variables and the observed variables can be described by a specific formulation.
However, we consider that such assumptions on unobserved variables are hard to guarantee, even if the domain knowledge on the dataset is available. Therefore, we aim to deal with the Reichenbach problem without assuming such a specificity on the possible relations between unobserved confounder and observed variables. 

In our previous study \cite{kobayashi2022detection}, we proposed \textsf{CLOUD} (CodeLength-based methOd for Unobserved common causes between Discrete data) to address the Reichenbach problem for discrete data. In this paper, we extend \textsf{CLOUD} to accomodate all types of data including discrete, mixed, and continuous. Unlike all existing methods, \textsf{CLOUD} does not specify a form of unobserved confounders.
We take a strategy in which we select a causal model which yields the minimum codelength among all candidates, known as the minimum description length (MDL) principle \cite{rissanen1978modeling}.
The key for our method is to employ normalized maximum likelihood (NML) code to compute the codelength in models of different capacities. We show this method exhibits high performance both in theoretically and experimentally.

The rest of this paper is organized as follows. The next section introduces the MDL principle and Structural Causal Models (SCMs), and then defines the Reichenbach problem formally.
In section~\ref{sec3}, we review existing methods from the viewpoint of two approaches and then see how it is hard to deal with the Reichenbach problem without assumptions in both approaches.
In section~\ref{sec4}, we describe models for which we consider the NML code and proposed method.
In section~\ref{sec5} and section~\ref{sec6}, we conduct theoretical analysis and extensive experiments, respectively. Finally, the conclusion is given in section~\ref{sec7}.

\section{Preliminaries}\label{sec2}

We introduce three theoretical frameworks upon which we construct our method for causal discovery, namely the minimum description length (MDL) principle, Structural Causal Models (SCMs) and the Reichenbach problem.

\subsection{The MDL principle and the NML Codelength}

The Minimum Description Length (MDL) principle is a model selection principle grounded in the concept of data compression. It asserts that the optimal model is the one that most succinctly describes both the data $x^n$ and the model $M$, where $x^n=(x_i)_{i=1, \ldots, n}$ is a data sequence of length $n$. 
In this principle, we compute the codelength of data with a universal code for each model, as an information criteria. 

In this context, we introduce the Normalized Maximum Likelihood (NML) code as a universal code.
The NML code is justified by the fact that, it is optimal in terms of the minimax regret criterion \cite{rissanen2012optimal} and that it exhibits consistency in model selection \cite{rissanen1989stochastic}.
The NML codelength, also known as stochastic complexity, is derived from the NML distribution. 
The NML distribution for statistical model $M$ with respect to data $x^n$ is defined as follows:
\begin{equation}
P_{\mathrm{NML}}\left(x^{n}; M\right)=\frac{P\left(x^{n} ; M, \hat{\bm{\theta}}(x^n)\right)}{\sum_{X^{n}}  P\left(X^{n} ; M, \hat{\bm{\theta}}(X^n)\right)},
\end{equation}
where $\hat{\bm{\theta}}(x^n)$ denotes the maximum likelihood estimator of the parameters of the model $M$ given the data $x^n$ and the summation $\sum_{X^{n}}$ is taken over the space of all the possible values of the data $X^n\in \mathcal{X}^n$.

The stochastic complexity, or the NML codelength, of $x^n$ is the negative logarithmic likelihood of the NML distribution:
\begin{align}
\mathcal{SC}(x^n; M) 
\coloneqq& -\log{P_{\mathrm{NML}}\left(x^{n}; M\right)} \notag \\
=& -\log P\left(x^{n} ; M,\hat{\bm{\theta}}\left(x^{n}\right)\right)+\log \sum\limits_{X^n} P\left(X^{n} ;M, \hat{\bm{\theta}}\left(X^{n}\right)\right)
\label{stochastic_complexity}
\end{align}

The first term of Eq. \eqref{stochastic_complexity} is the negative maximum log-likelihood, which is efficiently computable for many models. The second term, referred to as parametric complexity, is expressed as follows:
\begin{align}
\label{parametric_complexity}
    \log \mathcal{C}_n(M) \coloneqq \log \sum\limits_{X^n} P\left(X^{n} ;M, \hat{\bm{\theta}}\left(X^{n}\right)\right) 
\end{align}

The parametric complexity is not always analytically tractable, and its computation is one of the major focuses of the NML-based model selection. Techniques for its computation include deriving asymptotically consistent approximations \cite{rissanen2012optimal} and employing the g-function \cite{hirai2013efficient}.

As we see in below, the parametric complexity of a categorical distribution model with $K$ categories, denoted by $\log \mathcal{C}_{\mathsf{CAT}}(K, n)$, is represented by the following equation:
\begin{align*}
\label{multi_param_comp}
\mathcal{C}_{\mathsf{CAT}}(K, n) 
= &  \sum_{X^n\in \cbr{1,\ldots,K}^n}  \prod_{k=1}^{K}\left(\frac{n(X=k)}{n}\right)^{n(X=k)}.
\end{align*}
Here, $n(X=k)$ is the frequency of occurrence of value $k$ in the sequence $x^n$, defined as $n(X = k) = \sum_{i=1}^{n}\deffunc(x_{i} = k)$, where $\mathbb{I}$ is the indicator function. Kontkanen and Myllymäki developed an efficient recurrence formula for this model with a linear time complexity of $\mathcal{O}(n+K)$ \cite{kontkanen2007linear}:
\begin{align*}
    \mathcal{C}_{\mathsf{CAT}}(K=1, n) &= 1, \\ 
    \mathcal{C}_{\mathsf{CAT}}(K=2, n) &= \sum_{h_{1}+h_2=n} \frac{n !}{h_{1} ! h_{2} !} \left(\frac{h_{1}}{n}\right)^{h_{1}} \left(\frac{h_{2}}{n}\right)^{h_{2}}, \\
    \mathcal{C}_{\mathsf{CAT}}(K+2, n) &= \mathcal{C}_{\mathsf{CAT}}(K+1, n)+\frac{n}{K} \mathcal{C}_{\mathsf{CAT}}(K, n). 
\end{align*}

In dealing with continuous variables, the summation in Eq. \eqref{parametric_complexity} is replaced with an integral. 
Lastly, we illustrate how to compute the NML codelength for some statistical models. These examples cover data sequences of either discrete-type or continuous-type.
\begin{example}[The NML Codelength for a Discrete Data]
\label{ex:Example_Discrete} 
We compute the NML codelength for a data sequence $x^n$ under a categorical distribution model ${\mathsf{CAT}}^{m_X}$ with $m_X$ categories:
\begin{align*}
    \mathsf{CAT}^{m_X} = \left\lbrace P(X; \bm{\theta}) \ \middle| \ \bm{\theta} = (\theta_{0}, \dots, \theta_{m_X-1}), \theta_k \geq 0,  \sum_{k=0}^{m_X-1}\theta_{k}= 1 \right\rbrace.
\end{align*}
For given data $x^n$, the maximum likelihood estimator of parameter $\hat{\bm{\theta}} = (\hat{\theta}_{0}, \dots, \hat{\theta}_{m_X-1})$ is $\hat{\theta}_k(x^n)=\frac{n(X = k)}{n}$ for $k = 0, \cdots, m_X - 1$. Consequently, the first term in Eq. \eqref{stochastic_complexity} is computed as
\begin{align*}
-\log P(x^n; \mathsf{CAT}^{m_X}, \hat{\bm{\theta}}(x^n)) = -\sum_{k=0}^{m_X - 1}n(X=k) \log \frac{n(X=k)}{n}.
\end{align*}
The second term in Eq. \eqref{stochastic_complexity}, the parametric complexity of $\mathsf{CAT}^{m_X}$, is given by $\log \mathcal{C}_{\mathsf{CAT}}(K=m_X, n)$. Thus, the NML codelength for the discrete data $x^n$ under the $\mathsf{CAT}^{m_X}$ is given by
\begin{align}
\mathcal{SC}(x^n; \mathsf{CAT}^{m_X}) =  - \sum_{k=0}^{m_X - 1}n(X=k) \log \frac{n(X=k)}{n} + \log \mathcal{C}_{\mathsf{CAT}}(K=m_X, n).
\end{align}
\end{example}

\begin{example}[The NML codelength for a Continuous Data]
\label{ex:Example_Continuous}
Let the domain of continuous data sequence $x^n$ is $\Xcal = [0, 1)$.
We divide $\Xcal$ into $m_X$ cells $\lbrace I^{X}_k \rbrace$ of equal length $\frac{1}{m_X}$:
\begin{align*}
    \Xcal = \cup_{k=0}^{m_X - 1} I^{X}_k, \ I^{X}_k = \left[\frac{k}{m_X}, \frac{k+1}{m_X}\right) (k=0, \ldots, m_X-1)
\end{align*}
We define the histogram density function model $\mathsf{HIS}^{m_X}$ as follows:
\begin{align*}
    \mathsf{HIS}^{m_X} \!= \Bigg\{  p(X; \bm{\theta})=\!\! \sum_{k=0}^{m_X-1}\! \theta_{k} \mathbb{I}[X \in I^{X}_k] \,\Bigg| \, \bm{\theta} = (\theta_{0}, \dots, \theta_{m_X-1}), \theta_{k} \geq 0, \!\sum_{k=0}^{m_X-1}\!\frac{\theta_k}{m_X} = 1\Bigg\}.
\end{align*}

The maximum likelihood estimator for this model results in $\hat{\theta}_{k}(x^n) = \frac{n(X \in I^{X}_k)}{n} m_X$. Consequently, the maximum log-likelihood of data is calculated as:
\begin{align*}
    \log p(x^n; \hat{\bm{\theta}}(x^n)) 
    &= \sum_{k=0}^{m_X-1} n(X \in I^{X}_k)\log \hat{\theta}_{k}\\
    &= \sum_{k=0}^{m_X-1} n(X \in I^{X}_k) \log \frac{n(X \in I^{X}_k)}{n} + n \log m_X, 
\end{align*}
where $n(X \in I^{X}_k)$ is the frequency of observations in interval $I^{X}_k$ in the data sequence $x^n$, formally defined as $n(X \in I^{X}_k) = \sum_{i=1}^{n}\deffunc(x_{i^{}} \in I^{X}_k)$. 
Let
$\tilde{x}^n=(\tilde{x}_1,\ldots \tilde{x}_n)\in \{0,\ldots,m_X-1\}^n$ be a sequence of $n$ bin labels and
$I_{\tilde{x}^n}^X\coloneqq I_{\tilde{x}_1}^X\times \cdots \times I_{\tilde{x}_n}^X\subset [0, 1)^n$ be the $n$-dimensional hyper-bin associated with $\tilde{x}^n$.
The parametric complexity of $\mathsf{HIS}^{m_X}$ is then given by:
\begin{align*}
    \log \mathcal{C}_n (\mathsf{HIS}^{m_X})
    &= \log\int_{[0,1)^n}  p(x^n; \hat{\bm{\theta}}(x^n)) dx^n\\
    &= \log\sum_{\tilde{x}^n\in \{0,\ldots,m_X-1\}^n} \int_{I_{\tilde{x}^n}^X}  p(x^n; \hat{\bm{\theta}}(x^n)) dx^n\\
    &= \log \sum_{\tilde{x}^n \in \lbrace 0, \ldots, m_X-1 \rbrace^n} \prod_{k=0}^{m_X-1} \rbr{\frac{n(\tilde{X}=k)}{n}}^{n(\tilde{X}=k)} \\
    &= \log \mathcal{C}_{\mathsf{CAT}}(K=m_X, n),
\end{align*}
which results in the same value as the parametric complexity of the $m_X$-valued categorical model $\mathsf{CAT}^{m_X}$. 

The NML codelength thus becomes
\begin{align}
    & \mathcal{SC}(x^n; \mathsf{HIS}^{m_X}) \notag \\ 
    = & -  \sum_{k=0}^{m_X-1} n(X \in I^{X}_k)\log \frac{n(X \in I^{X}_k)}{n} - n \log m_X + \log \mathcal{C}_{\mathsf{CAT}}(K=m_X, n), \label{eq:L_HIS_mX}
\end{align}
for continuous data $x^n$ based on $\mathsf{HIS}^{m_X}$.
\end{example}


\subsection{Structural Causal Model}
A Structural Causal Model (SCM) \cite{pearl2009causality} represents the data-generating process through a set of structural assignments. In an SCM, variables are expressed as functions of their parent variables (direct causes) and exogenous variables. In particular, when we consider only two variables $X$ and $Y$, which are both one-dimentional, with the causal graphs $X\to Y$ or $X \gets Y$, SCMs can be represented as follows:
\begin{align*}
    & M_{X \to Y}:
    \begin{cases}
    X=E_{X}\\
    Y=f(X, E_Y)
    \end{cases}
    & M_{X \gets Y}:
    \begin{cases}
    X=g(Y, E_{X})\\
    Y=E_Y,
    \end{cases}
\end{align*}
where $f$ and $g$ are functions, and $E_X, E_Y$ are one-dimensional exogenous variables such that $E_X \Perp E_Y$. Here, the statistical models $M_{X \to Y}, M_{X \gets Y}$, derived from each SCM, are referred to as causal models. In general, without constraints on the distributions of the exogenous variables and/or on the forms of functions $f$ and $g$, it is not identifiable whether samples from the joint distribution $P(X, Y)$ are induced by the causal relationship of $M_{X \to Y}$ or $M_{X \gets Y}$ \cite[Proposition 4.1]{peters2017elements}. In other words, causal discovery from observational data usually requires making specific assumptions on the functional forms and/or the distributions of exogenous variables, which limits the scope of the joint distributions led by SCMs. Some notable SCMs in the context of causal discovery are listed in below.
\begin{example}[Additive Noise Model ($\textsf{ANM}$)]
    \label{ex:ANM}
    \textsf{ANM}\cite{hoyer2008nonlinear, peters2010identifying, peters2014causal} assumes a data generating process as per the following equations where effects are the nonlinear functions of their causes with additive noise:
    \begin{align*}
        & M_{X \to Y}:
        \begin{cases}
        X=E_{X}\\
        Y=f(X) + E_Y,
        \end{cases}
        & M_{X \gets Y}:
        \begin{cases}
        X=g(Y) + E_{X}\\
        Y=E_Y,
        \end{cases}
    \end{align*}
    where additive noises $E_X, E_Y$ satisfy $E_X \Perp Y, E_Y \Perp X$, respectively.
\end{example}

\begin{example}[Linear NonGaussian Acyclic Model ($\textsf{LiNGAM}$)]
    \label{ex:LiNGAM}
    \textsf{LiNGAM}\cite{shimizu2006linear, shimizu2011directlingam}, which is a special case of \textsf{ANM}, assumes that causal relationships are linear and that exogenous variables follow non-Gaussian distributions: 
    \begin{align*}
        & M_{X \to Y}:
        \begin{cases}
        X=E_{X}\\
        Y=b_{X \to Y} X + E_Y,
        \end{cases}
        & M_{X \gets Y}:
        \begin{cases}
        X=b_{X \gets Y}Y + E_{X}\\
        Y=E_Y,
        \end{cases}
    \end{align*}
    where $b_{X \to Y}, b_{X \gets Y} \in \mathbb{R}$ are the linear coefficients, and $E_X, E_Y$ follow Non-Gaussian distributions.
\end{example}

\begin{example}[Linear Mixed causal model ($\textsf{LiM}$)]
    \label{ex:LiM}
    $\textsf{LiM}$ \cite{zeng2022causal} is an extension of \textsf{LiNGAM} to accommodate mixed data types, including both continuous and discrete variables. In the case of continuous variables, $\textsf{LiM}$ assumes the same SCMs as $\textsf{LiNGAM}$. When $X$ is a continuous variable and $Y$ is a binary variable, $\textsf{LiM}$ assumes the following SCM: 
    \begin{align*}
        & M_{X \to Y}:
        \begin{cases}
        X=E_{X}\\
        Y=\begin{cases}
            1 \ (b_{X \to Y} X + E_Y > 0), \\
            0 \ (\text{otherwise}), \\
          \end{cases} \\
        \end{cases}\\
    \end{align*}
    where $E_Y$ follows a Logistic distribution.
\end{example}

\begin{example}[LiNGAM with latent confounder ($\textsf{lvLiNGAM}$)]
    \label{ex:LiNGAMwlc}
    \textsf{lvLiNGAM} extends the basic \textsf{LiNGAM} model to incorporate hidden common causes \cite{hoyer2008estimation}. Besides the linear and non-Gaussian assumptions of \textsf{LiNGAM}, \textsf{lvLiNGAM} explicitly models unobserved common causes $C$. For a one-dimensional $C$, the model can be represented with the following SCM:
    \begin{align*}
        & M_{X \gets C \to Y}:
        \begin{cases}
        X=\lambda_X C + E_{X}\\
        Y=\lambda_Y C + E_Y,\\
        \end{cases}
    \end{align*}
    where $\lambda_X, \lambda_Y \in \mathbb{R}$ denote the direct causal effects from the unobserved common cause $C$ to each observed variable to the observed variables $X$ and $Y$, respectively. In this SCM, the latent variable $C$ is assumed to be non-Gaussian and independent of $E_X$ and $E_Y$, and it is assumed to have linear effects on the observed variables.
\end{example}

\subsection{Reichenbach Problem}
This section describes the Reichenbach problem formally, which is central to our study.
Suppose we have i.i.d.\ observational data $z^n = (x^n , y^n) \in \Xcal^n \times \Ycal^n$ generated from joint distribution $P(X, Y)$.
Here, $X$ and $Y$ can be either discrete or continuous variables. 

Based on Reichenbach's common cause principle, we can categorize the causal relationship between $X$ and $Y$ into four cases. The goal is to infer one of these causal models $M$ that best explains the underlying causal relationship:
\begin{itemize}[itemsep=2pt]
\item $M_{X \Perp Y}$: $X$ and $Y$ are independent, with no direct causal link.
\item $M_{X \gets C \to Y}$: There exist common causes $C$ that causally influence both $X$ and $Y$.
\item $M_{X \to Y}$: $X$ causes $Y$, but not vice versa.
\item $M_{X \gets Y}$: $Y$ causes $X$, but not vice versa.
\end{itemize}

In solving the Reichenbach problem, it is desirable to use methods that do not make any assumption about  unobserved variables $C$.  This is crucial because it is almost impossible to have prior knowledge of all potential unobserved variables. Therefore, we propose a method capable of detecting the presence of unobserved common causes, even in situations where the candidates for latent variables are unknown or when unexpected unobserved variables are present, without relying on assumptions about $C$.

To employ existing causal discovery methods that rely on modeling unobserved variables, one needs to know the nature of the unobserved causes beforehand. However, it is challenging to acquire complete knowledge about all possible unobserved variables solely from the domain knowledge of the observed variables. Thus, when applying these methods to real-world data, a heightened level of caution is required. Nevertheless, existing approaches to the Reichenbach problem typically depend on models that make assumptions about the relationships between unobserved variables $C$ and the observed variables $X, Y$.

\section{Existing Work}\label{sec3}
Existing causal discovery methods from a joint distribution of two variables can be categorized into two approaches: one employs the identifiability of the model and the other is based on  the principle of algorithmic independence of conditionals. Furthermore, in the context of the Reichenbach problem, there are methods focused on solving the sub-problem of choosing between $X\to Y$ and $X \gets Y$, and those that attempt to solve the Reichenbach problem by making assumptions about unobserved common causes. In this section, we describe these approaches and discuss why it is difficult to detect unobserved common factors without making assumptions about these causes.

\subsection{Identifiable models}\label{subsec1}
In this approach, when we formulate causal relationships using SCMs, we restrict the functional forms and the distributions of the exogenous variables.
The causal models induced by SCMs are said to be identifiable if different causal structures always lead to different joint distributions of the observed variables. 

In order to infer a causal structure using identifiable models, we assume that the observed data have been generated by a distribution belonging to one of these models.
Then, we infer $X$ is the cause of $Y$ when the corresponding model explains the data best of all models. This reasoning applies similarly to other causal relationships.

To determine the causal direction,
various type of indentifiable models are studied in existing work so that various types of data can be applied.
Shimizu et al. \cite{shimizu2006linear} showed that causal models become identifiable if function is linear and the distribution of exogenous variables is non-Gaussian as in Example~\ref{ex:LiNGAM} (\textsf{LiNGAM}). 
A general case of \textsf{LiNGAM} is Additive Noise Models (\textsf{ANM}s, \cite{hoyer2008nonlinear}), as detailed in Example~\ref{ex:ANM}, where we assume that the noise is additive and independent of the cause. In general, \textsf{ANM} is identifiable if functional form is non-linear even without imposing any restrictions on the noise distributions~\cite{hoyer2008nonlinear, peters2014causal}. 
\cite{peters2010identifying} proposed \textsf{DR}, which is an extension of \textsf{ANM}s to discrete variables and showed that \textsf{ANM} is generally identifiable in discrete case. 
For mixed-type data, \cite{zeng2022causal} formulated Linear Mixed causal model (\textsf{LiM}) as shown in Example~\ref{ex:LiM}, and Li et al. \cite{li2022hybrid} proposed an algorithm, \textsf{HCM}, which formulates nonlinear causal relationships as mixed-SCMs. 

In this approach, we can recover the true causal relationship based on observed data as long as there is no latent confounder, i.e., the true distribution belongs to one of the models we adopt. Therefore, this framework requires practitioners to make use of domain knowledge such that causal relationship, if any, can be formulated in a certain type of SCM.
However, causal discovery methods under the assumption of causal sufficiency, which assume no unobserved common causes, can lead to incorrect conclusions when applied to data where unobserved common causes actually exist. To avoid such issues, it becomes crucial to consider unobserved common causes in causal discovery.

As for models that allow for the presence of unobserved common causes, Hoyer et al.  proposed a model called \textsf{lvLiNGAM}~\cite{hoyer2008estimation}, which extends \textsf{LiNGAM} to the models with latent counfounding variables by explicitly modeling them. \textsf{lvLiNGAM}, shown in Example~\ref{ex:LiNGAMwlc}, assumes that the latent confounder $C$ follow non-Ganssian distributions and relationships between $C$ and observed variables are linear. 
Parce LiNGAM (\textsf{BUPL}, \cite{tashiro2014parcelingam}) and Repetitive Causal Discovery (\textsf{RCD}, \cite{maeda2020rcd}) make the same assumptions on SCMs. Maeda et al. extended \textsf{lvLiNGAM} to its non-linear variant, and proposed \textsf{CAMUV} algorithm~\cite{maeda2021causal} . 

For models to be identifiable including $M_{X \gets C \to Y}$, one must make an assumption on the model of unobserved common causes. This means that one must have known the nature of the unobserved common causes beforehand to successfully perform causal discovery. Otherwise, conclusions led by this framework will be unreliable when unobserved common causes do not follow the assumptions. Even with domain knowledge of $X$ and $Y$, it remains challenging to accurately determine the form of SCM for an unobserved common cause, which a practitioner is not certain if exists, will satisfy.

\subsection{Algorithmic Independence of Conditionals}
This section describes the approach based on the principle of algorithmic independence of conditionals, which is described as follows: if true causality is $X \to Y$, 
then mechanism $P^*(Y|X)$ is independent of the cause $P^*(X)$ \cite{janzing2010causal}, where $P^*$ denotes true distributions we assume under the corresponding causal relationship. 
By denoting the Kolmogorov complexity as $K$, it leads to the following inequality: $$K(P^*(X)) + K(P^*(Y|X)) < K(P^*(Y)) + K(P^*(X|Y)),$$
if true causality is $X \to Y$~\cite{stegle2010probabilistic}.
This inequality can not be evaluated due to the following two reasons: Kolmogorov complexity is not computable and true distribution is unknown. Therefore, Marx and Vreeken~\citep{marx2021formally} has justified that this principle leads to the approximation build upon two-part MDL as follows:
\begin{align*}
    \Lcal(z^n; M_{X \to Y}) < \Lcal(z^n; M_{X \gets Y}),
\end{align*}
where $\Lcal(z^n; M_{X \to Y})$ is the description length of data $z^n$ under statistical model $M_{X \to Y} = M_X\times M_{Y|X}$ in which we assume $P^*(X) \in M_X$ and $P^*(Y|X) \in M_{Y|X}$, formally defined as follows:
\begin{align*}
     \Lcal(z^n; M_{X \to Y}) = L(x^n; M_{X}) + L(y^n; x^n, M_{Y|X}).
\end{align*}
We define $\Lcal(z^n; M_{X \gets Y})$ analogously. 
Methods such as Causal Inference by Stochastic Complexity 
 (\textsf{CISC}, \cite{budhathoki2017mdl}), Accurate Causal Inference on Discrete data (\textsf{ACID}, \cite{budhathoki2018accurate}) and Distance Correlation (\textsf{DC}, \cite{liu2016causal}) model $\Lcal(z^n; M)$ using refined MDL, Shannon Entoropy and distance correlation, respectively.

If one considers including confounded model $M_{X \gets C \to Y}$ with unobserved common causes $C$, the description length under the joint distributions of that model, $\Lcal(z^n; M_{X \gets C \to Y})$ which is an approximation of $K(P^{*}(X, Y, C))$, must be evaluated and compared with directed cases of $M_{X\to Y}$ and $M_{X \gets Y}$. A naive approximation approach based solely on likelihood invariably leads to the selection of the confounded model $M_{X \gets C \to Y}$, due to its inherently minimized complexity. To address this issue, Confounded-or-Causal (\textsf{COCA}) method was developed \cite{kaltenpoth2019we}. \textsf{COCA} selects between $X \to Y$ and $X \gets C \to Y$ under the assumption that not only the observed variables but also unobserved common causes $C$ follow specific-dimensional Gaussian distributions.
It employs Bayesian coding to approximate the description length of data, considering not only the likelihood but also the complexity of the model class, including $C$. This approach enables to comparison between different sizes of statistical models, specifically $M_{X\to Y}$ and $M_{X \gets C \to Y}$. However, it is important to note that this approach relies on certain assumptions about $C$.
\newline

In Summary, in challenging the Reichenbach problem, all existing methods face a common limitation: they require additional assumptions about unobserved common causes. Our previous work has already shown that \textsf{CLOUD} successfully overcomes this limitation in the context of discrete variables~\cite{kobayashi2022detection}. We claimed that this is achievable by comparing models with different capacities, as quantified using the NML codelength. In this paper, our objective is to expand the applicability of \textsf{CLOUD} to encompass continuous and mixed data types, thereby enhancing its effectiveness in solving the Reichenbach problem across a wider range of data types.

\section{Proposed Method}\label{sec4}
In this section, we extend \textsf{CLOUD} to cover all data types, which was originally designed for the Reichenbach problem in discrete data. In \textsf{CLOUD}, we formulate causal models for four causal relationships in the Reichenbach problem (Section \ref{sec4.1}). Then, based on MDL principle, we calculate the codelength of the observed data $z^n$ based on NML coding and select a causal model $M$ which achieves the shortest codelength (Section \ref{sec4.2}). 

We formulate confounded model $M_{X \gets C \to Y}$ to represent any joint distribution, thus avoiding assumptions about $C$. Since this model has the highest complexity compared to the other three, 
the NML-based codelength on this model is not necessarily the shortest although the negative loglikilhood is the smallest.
Thus, by considering both model complexity and data likelihood, we can select an appropriate causal model among models of different complexities.


\subsection{Model}\label{sec4.1}
In this section, we describe causal models for cases where both $X$ and $Y$ are discrete or continuous variables, represented as statistical models derived from the assumed Structural Causal Models (SCMs) for each causal relationship.

First, we formulate SCMs for each causal model $M$ to describe the causal relationships between $X \in \Xcal$ and $Y \in \Ycal$:
\begin{align*}
&M_{X \Perp Y}:
\begin{cases}
X=E_{X}\\
Y=E_{Y}
\end{cases}
&M_{X \gets C \to Y}:         
\begin{cases}
X=f(C, E_{X}) \\ 
Y=g(C, E_{Y})
\end{cases} \\
&M_{X \to Y}:
\begin{cases}
X=E_{X} \\
Y=f(X) + E_Y 
\end{cases}
&M_{X \gets  Y}: 
\begin{cases}
X=g(Y) + E_{X}\\
Y=E_Y 
\end{cases}
\end{align*}
Here, the exogenous variables $E_X \in \Xcal, E_Y \in \Ycal$ are independent of each other. 
For $M_{X \Perp Y}$, we assume faithfullness~\cite{spirtes2000causation} in the sense that we regard $X$ and $Y$ is causally independent when they are statistically independent. 

For $M_{X \gets C \to Y}$, any probability density functions on $(\mathbb{Z} / m_X\mathbb{Z})\times (\mathbb{Z} / m_Y\mathbb{Z})$ and probability density functions on $(\mathbb{R} / \mathbb{Z})^2$ belong to the $M_{X \gets C \to Y}$ by considering appropriate choice of $f$, $g$ and $C$. In this sense, we do not impose assumptions on the unobserved common cause $C$ for confounded case. 

For $M_{X \to Y}$ and $M_{X \gets Y}$, we assume additive noise models (\textsf{ANMs}) \cite{peters2010identifying}, which employs functions from the function sets $\Fcal = \lbrace f: \Xcal \to \Ycal \mid f \text{ is not constant} \rbrace$ and $\Gcal = \lbrace g: \Ycal \to \Xcal \mid g \text{ is not constant}\rbrace$.  In case $X$ is discrete, we set $\Xcal = \lbrace 0, 1, \dots, m_X - 1 \rbrace$ and addition is taken over $\mathbb{Z} / m_X\mathbb{Z}$ as in \cite{peters2017elements, peters2011causal}. In case $X$ is continuous, we set $\Xcal =  [0,1)$ and addition is taken over $\mathbb{R} / \mathbb{Z}$. The same goes for cases $Y$ is either discrete or continuous. In continuous cases, the addition can be regarded as addition over $\mathbb{R}$ for data scaled with a sufficiently small constant $\epsilon$. This implies that, in practical applications, models with periodic boundary conditions can be considered as including non-periodic \textsf{ANMs}. 

Second, we identify the causal models $M_{X\Perp Y}$, $M_{X\gets C \to Y}$, $M_{X\to Y}$
and $M_{X\gets Y}$ with a set of joint probability distributions on $(X,Y)$ that models imply. It can be justified in case in which only observations from the joint distribution are available.

\paragraph{Discrete Case: }
If both $X$ and $Y$ are discrete, then $\Xcal=\ZZ/m_X\ZZ$ and $\Ycal=\ZZ/m_Y\ZZ$.
Any discrete probability distribution on $(X,Y)$ can be identified with a parameter $\thetab \in \bm{\Theta}$ by the following relation:
\begin{align*}
P(X, Y ; \bm{\theta}) = \prod_{k, k'}\theta_{k, k'}^{\deffunc[X=k, Y=k']},
\end{align*}
where we define $\bm{\Theta} = \cbr{ \thetab = (\theta_{k, k'})_{k,k'} \in \mathbb{R}^{m_X\times m_Y} \middle| \theta_{k,k'} \ge 0, \sum \theta_{k,k'} =1 }$.
Based on this parametrization, we represent causal models by characterizing the respective subset of the parameter space as follows:
\begin{align*}
    M_{X \Perp Y}&=\left\{P\left(X, Y  ; \thetab \right) = \prod_{k, k'}\theta_{k, k'}^{\deffunc[X=k, Y=k']} \,\middle|\, \thetab\in \bm{\Theta}_{X \Perp Y} \right\}, \\
    M_{X \gets C \to Y}&= \left\{P\left(X, Y  ; \thetab \right) =\prod_{k, k'}\theta_{k, k'}^{\deffunc[X=k, Y=k']} \,\middle|\, \thetab\in \bm{\Theta}_{X \gets C \to Y} \right\}, \\
    M_{X \to Y}&=\left\{P\left(X, Y  ; \thetab\right) = \prod_{k, k'}\theta_{k, k'}^{\deffunc[X=k, Y=k']} \,\middle|\, \thetab\in \bm{\Theta}_{X \to Y} \right\} ,\\
    M_{X \gets Y}&=\left\{P\left(X, Y  ; \thetab \right)=\prod_{k, k'}\theta_{k, k'}^{\deffunc[X=k, Y=k']} \;\middle|\, \thetab\in \bm{\Theta}_{X \gets Y} \right\},
\end{align*}
where
\begin{subequations}\label{eq:Theta_discrete}
    \begin{align}
    \bm{\Theta}_{X \Perp Y} &= \cbr{ (\theta^{X}_k\,\theta^{Y}_{k'})_{k, k'} \in \bm{\Theta} \,\middle|\,  \sum_k \theta^{X}_k = 1, \theta^{X}_k \ge 0, \sum_{k'} \theta^{Y}_{k'} = 1, \theta^{Y}_{k'} \ge 0 }, \\
    \bm{\Theta}_{X \gets C \to Y} &= \bm{\Theta} ,\\
    \bm{\Theta}_{X \to Y} &= \cbr{ (\theta^{X}_k\,\theta^{Y}_{f(k)+k'})_{k, k'} \in \bm{\Theta} \,\middle|\,  \sum_k \theta^{X}_k = 1, \theta^{X}_k \ge 0, \sum_{k'} \theta^{Y}_{k'} = 1, \theta^{Y}_{k'} \ge 0, f \in \Fcal},\\
    \bm{\Theta}_{X \gets Y} &= \cbr{ (\theta^{X}_{g(k')+ k}\,\theta^{Y}_{k'})_{k, k'} \in \bm{\Theta} \,\middle|\,  \sum_k \theta^{X}_k = 1, \theta^{X}_k \ge 0, \sum_{k'} \theta^{Y}_{k'} = 1, \theta^{Y}_{k'} \ge 0, g \in \Gcal}.
\end{align}
\end{subequations}
Note that the addition in subscripts is taken over each respective finite space. For discrete case, we encode the data $z^n$ based on these discrete causal models.

\paragraph{Continuous Case: } 
In case the domains of $X$ and $Y$ are both continuous, we represent their joint density function by the infinite union of histogram densities.

Firstly, we consider partitioning $\Xcal$ into $m_X$ equally-sized cells $\lbrace I^{X}_k \rbrace \ (k=0, \ldots, m_X-1)$ and $\Ycal$ into $m_Y$ equally-sized cells $\lbrace I^{Y}_{k'} \rbrace \ (k'=0, \ldots, m_Y-1)$.

We then define the two-dimensional density function model $\mathsf{HIS}^{m_X, m_Y}$ as follows:
\begin{align*}
    \mathsf{HIS}^{m_X, m_Y} \!= \cbr{  p(X, Y; \bm{\theta})=\! \sum_{k=0}^{m_X-1}\sum_{k'=0}^{m_Y-1}\! \theta_{k,k'} \mathbb{I}[X \in I^{X}_k, Y \in I^{Y}_{k'}] \,\middle|\, \thetab \in \bm{\Theta}^{m_X,m_Y} }
\end{align*}
where 
\begin{align*}
    \bm{\Theta}^{m_X,m_Y} = \cbr{\thetab= (\theta_{k,k'})_{k,k' } \in \mathbb{R}^{m_X\times m_Y}\,\middle| \, \theta_{k,k'} \geq 0, \!\sum_{k=0}^{m_X-1}\sum_{k'=0}^{m_Y-1}\!\frac{\theta_{k,k'}}{m_X m_Y} = 1}.
\end{align*}
Then, the infinite union $\bigcup_{m_X,m_Y}\mathsf{HIS}^{m_X, m_Y}$ can represent any joint density function n an arbitrary precision level. Due to the universality of $M_{X\gets C \to Y}$, we represent
\begin{align*}
M_{X \gets C \to Y}=\bigcup_{m_X, m_Y} M^{m_X, m_Y}_{X \gets C \to Y},\quad 
M^{m_X, m_Y}_{X \gets C \to Y} = \mathsf{HIS}^{m_X, m_Y}.
\end{align*}
It means that we can identify any joint distribution on $(X,Y)$ by an element of $M_{X\gets C\to Y}$, in the sense that it gives the same distribution of codelength for a given code when the approximation level is fixed. 

By defining $\bm{\Theta}^{m_X,m_Y}_{X\Perp Y}, \bm{\Theta}^{m_X,m_Y}_{X\to Y}, \bm{\Theta}^{m_X,m_Y}_{X\gets Y} \subset \bm{\Theta}^{m_X,m_Y}$ similarly as equations~\eqref{eq:Theta_discrete}, 
we can represent the other models as follows:
\begin{align*}
    M_{X \Perp Y}=\bigcup_{m_X, m_Y} M^{m_X, m_Y}_{X \Perp Y},&\quad M^{m_X, m_Y}_{X \Perp Y} = \cbr{ p(X, Y; \bm{\theta}) \,\middle|\, \thetab \in \bm{\Theta}^{m_X,m_Y}_{X\Perp Y} }   \\
    M_{X \to Y}=\bigcup_{m_X, m_Y} M^{m_X, m_Y}_{X \to Y},&\quad M^{m_X, m_Y}_{X \to Y} = \cbr{ p(X, Y; \bm{\theta}) \,\middle|\, \thetab \in \bm{\Theta}^{m_X,m_Y}_{X\to Y} }   \\ 
    M_{X \gets Y}=\bigcup_{m_X, m_Y} M^{m_X, m_Y}_{X \gets Y},&\quad M^{m_X, m_Y}_{X \gets Y} = \cbr{ p(X, Y; \bm{\theta}) \,\middle|\, \thetab \in \bm{\Theta}^{m_X,m_Y}_{X\gets Y} }.   
\end{align*}
We can regard $M_{X\Perp Y}$ as a set of any density function that is decomposable as $p(X,Y)=p(X)p(Y)$.
As for $M_{X\to Y}$, we can see it as a special form of decomposition of the density function such as $p(X, Y)=p(X)p(Y-f(X))$ using a function $f$.




\subsection{Algorithm}\label{sec4.2}
We regard the Reichenbach problem as a problem of model selection and conduct a selection under the MDL criterion. That is, among the causal models defined in Section \ref{sec4.1}, we infer the underlying causal relationship of the data $z^n$ is such a causal model $M$ that yields the shortest description length of the data.

\paragraph{Discrete Case:} 
\begin{algorithm}
\caption{Main function in Discrete Case}
\label{alg:main_discrete_case}
\begin{algorithmic}[1]
\Require Data \( z^n \), A set of model candidates \( \mathcal{M} \)
\Ensure Best model \( \hat{M} \)
\For{\( M \) in \( \mathcal{M} \)}
    \State Compute \( \mathcal{L}^{d}(z^n; M) \) using \eqref{eq:L_disc_all}
\EndFor
\State \( \hat{M} = \mathrm{argmin}_{M\in\Mcal} \mathcal{L}^{d}(z^n; M) \)
\State \textbf{return} \( \hat{M} \)
\end{algorithmic}
\end{algorithm}
We infer a causal relationship by selecting a discrete causal model according to the following equation:
\begin{align}\label{alg_disc}
    \Hat{M} (z^n)= \argmin_{M \in \mathcal{M}} \mathcal{L}^{d}(z^n; M),
\end{align}
where $\mathcal{L}^{d}(z^n; M)$ is a universal codelength of the discrete data $z^n$ for the discrete causal model $M$, and we employ NML code to compute the codelength. $\mathcal{M}$ is a set of model candidates, $\Mcal = \cbr{M_{X \to Y}, M_{X \gets Y}, M_{X \Perp Y}, M_{X \gets C \to Y}}$. We can also set $\mathcal{M}$ to be its subset, such as $\cbr{M_{X \to Y}, M_{X \gets Y}}$, based on the prior knowledge.
The algorithm for discrete case is shown in Algorithm \ref{alg:main_discrete_case}.

We calculate $\mathcal{L}^{d}(z^n; M)$ for $M_{X \Perp Y}$ and $M_{X \gets C \to Y}$ using exact NML codes. For $M_{X \to Y}$ and $M_{X \gets Y}$, we employ two-stage coding based on the NML code since it is hard to exactly calculate the codelength of the NML code because functions $f$ and $g$ are not fixed. For $M_{X \to Y}$, we compute the codelength based on two-stage coding with respect to function $f$ as follows:
\begin{align}
\label{eq:disc_x2y_two_part_code}
\Lcal^{d}(z^n;M_{X\to Y}) = L(f; M_{X\to Y}) + L(z^n; M_{X\to Y}, f ),
\end{align}
The first term on the right-hand side is a codelength required to encode a function $f$, and the second term represents the NML codelength for the model $M_{X \to Y}$ with function $f$ fixed. The same applies to $M_{X\gets Y}$.
The forms of each codelength are provided in Proposition~\ref{prop1}. We provide its proof in Appendix \ref{secB1}.

\begin{prop}[NML-based codelength for discrete data]\label{prop1}
For a given discrete data $z^n$ and the discrete causal models $M\in\cbr{M_{X \Perp Y},M_{X \to Y},M_{X \gets Y},M_{X \gets C \to Y}}$, the codelengths defined as above have the following expressions:
\begin{align}
    &\mathcal{L}^{d}(z^n ; M) \notag \\
    &=\begin{dcases}
    \ell_X+\ell_Y + \log\left(\mathcal{C}_{\mathsf{CAT}}(m_X, n) \cdot \mathcal{C}_{\mathsf{CAT}}(m_Y, n)\right) & \textnormal{if }M=M_{X \Perp Y} \text{,}\\
    \ell_{X,Y} + \log  \mathcal{C}_{\mathsf{CAT}}(m_X  m_Y, n) & \textnormal{if }M=M_{X \gets C \to Y}\text{,} \\
    \ell_X +\ell_{Y|X}(\fhat)+ \log\left(\mathcal{C}_{\mathsf{CAT}}(m_X, n) \cdot \mathcal{C}_{\mathsf{CAT}}(m_Y, n)\right) \\
    \quad + \log  (m_Y^{m_X - 1} - 1) & \textnormal{if }M=M_{X \to Y}\text{,} \\ 
    \ell_Y +\ell_{X|Y}(\ghat) + \log\left(\mathcal{C}_{\mathsf{CAT}}(m_X, n) \cdot \mathcal{C}_{\mathsf{CAT}}(m_Y, n)\right) \\
    \quad + \log (m_X^{m_Y - 1} - 1) & \textnormal{if }M=M_{X \gets Y},
    \end{dcases}\label{eq:L_disc_all}
\end{align}
where 
\begin{align*}
    \ell_X &= - \sum_{k=0}^{m_X - 1} n(X = k) \log \frac{n(X = k)}{n},\\
    \ell_{Y|X}(f) &= - \sum_{k'=0}^{m_Y - 1} n(Y=f(X) + k') \log \frac{n(Y=f(X) + k')}{n}, \\
    \ell_{X,Y} &=-\sum_{k=0}^{m_X - 1} \sum_{k'=0}^{m_Y - 1} n\left(X = k, Y=k'\right) \log \frac{n\left(X = k, Y = k'\right)}{n},
\end{align*}
and similarity for $\ell_{Y}$ and $\ell_{X|Y}$.  Here, $\fhat$ and $\ghat$ are functions derived through maximum likelihood estimation.
\end{prop}

Proposition \ref{prop1} shows that while $M_{X \gets C \to Y}$ is the most expressive model and the negative log-likelihood of the data is always minimized in $M_{X \gets C \to Y}$, its parametric complexity is the largest among all models. Therefore, by employing NML-based codelength as shown in Eq. \eqref{eq:L_disc_all}, we can compare between models with varies capacities as per Eq. \eqref{alg_disc} under the trade-off between data likelihood and model complexity.

The algorithm for estimation of function $\fhat$ or $\ghat$ is shown in Algorithm \ref{alg:optimize_regression} in Appendix \ref{secA1}. 

\paragraph{Continuous Case:} 
\begin{algorithm}
\caption{Main function in Continuous Case}
\label{alg:main_continuous_case}
\begin{algorithmic}[1]
\Require Data \( z^n \), a set of model candidates \( \mathcal{M} \), and candidates $\Pcal$ for bin numbers $m_X$ and $m_Y$
\Ensure Best model \( \hat{M} \)
\For{\( (m_X, m_Y)\) in \( \Pcal \)}
    \For{\( M \) in \( \mathcal{M} \)}
        \State Compute \( \mathcal{L}^{c}(z^n,m_X, m_Y; M) \)  using \eqref{eq:Lc_MmXmY}
    \EndFor
\EndFor
\State \textbf{return} \( \hat{M} = \argmin_{M \in \Mcal} \min_{(m_X, m_Y)\in \Pcal }\mathcal{L}^{c}(z^n,m_X, m_Y; M) \)
\end{algorithmic}
\end{algorithm}
We infer a causal relationship by selecting a continuous causal model $M$ according to the following equation:
\begin{align}
    \Hat{M} (z^n) &= \argmin_{M \in \mathcal{M}}  \min_{(m_X,m_Y) \in \Pcal} \mathcal{L}^{c}(z^n,m_X,m_Y; M).
\end{align}
Here, $\mathcal{L}^{c}(z^n,m_X,m_Y; M)$ is a codelength of data $z^n$ required to encode $z^n$ up to an arbitrarily precision $\delta>0$.
In order to construct a universal code of $z^n$ under causal model $M$, we employ two-stage coding for $m_X, m_Y$ as follows:
\begin{align}
\label{eq:Lc_2stg_mXmY}
    \Lcal^{c}(z^n,m_X,m_Y; M) &= L(m_X, m_Y; M) + L(z^n; M, m_X, m_Y),  
\end{align}
where $L(m_X, m_Y; M)$ is the codelength required to encode the numbers of bins $m_X$ and $m_Y$. 
The second term, $L(z^n; M, m_X, m_Y)$, represents codelengths for encoding $z^n$ based on continuous models of corresponding bin sizes, $M_{X \Perp Y}^{m_X,m_Y},M_{X\gets C\to Y}^{m_X,m_Y},M_{X\to Y}^{m_X,m_Y}$, or $M_{X\gets Y}^{m_X,m_Y}$.

In order to encode $z^n=(x^n,y^n)$ with the precision $\delta$, we again employ two-part coding through $\textnormal{disc}(x^n;m_X)$ and $\textnormal{disc}(y^n;m_Y)$. 
Here, we define $\textnormal{disc}: \Xcal^n \to \lbrace 0, \ldots, m_X-1 \rbrace^n$ as a function that discretizes continuous data $x^n \in \Xcal$ into $m_X$ equal categories.
Note that the discretized data is encoded by the strategy mentioned above as in discrete case.
After the discretized data is encoded, additional codelength are required to attain the prespecified precision level. This additional codelength is denoted by
$L(x^n; \textnormal{disc}(x^n; m_X))$ or $L(y^n; \textnormal{disc}(y^n; m_Y))$. Thus, the second term in Eq.~\eqref{eq:Lc_2stg_mXmY} has the following expression.
\begin{align}
\label{eq:Lc_2stg_discxnyn}
L(z^n; M, m_X, m_Y) =& \Lcal^{d}(\textnormal{disc}(x^n; m_X) , \textnormal{disc}(y^n; m_Y); \textnormal{DISC}(M; m_X, m_Y)) \notag \\
    &+ L(x^n; \textnormal{disc}(x^n; m_X)) + L(y^n; \textnormal{disc}(y^n; m_Y)),
\end{align}
where $\textnormal{DISC}(M; m_X, m_Y)$
denotes the discrete causal model of causality $M \in \mathcal{M}$ with the category numbers $m_X$ and $m_Y$. 
Consequently, the form of each codelength have an expression as provided in Proposition~\ref{prop2}.
We provide its proof in Appendix \ref{secC1}.
\begin{prop}[Codelength in Continuous Case]\label{prop2}
The codelength of data $z^n$ required to encode $z^n$ as described above has the following expression:
\begin{align}
\label{eq:Lc_MmXmY}
    \Lcal^{c}(z^n; M) & = \Lcal^{d}(\textnormal{disc}(x^n; m_X) , \textnormal{disc}(y^n; m_Y); \textnormal{DISC}(M; m_X, m_Y)) \\ \notag
    & \quad + L^{c \to d}(m_X, n) + L^{c \to d}(m_Y, n) + \textnormal{const.},
\end{align}
where the first term on the right-hand side in Eq. \eqref{eq:Lc_MmXmY} can be calculated using Eq. \eqref{eq:L_disc_all} in Proposition \ref{prop1} and
$L^{c \to d}(m,n)$ for $m\in \mathbb{N}^+$ is defined as:
\begin{align}
\label{L_c2d}
    L^{c \to d}(m, n) = - n \log m + \log^{*} m,
\end{align}
where $\log^{*} m$ is given by Rissanen's universal integer coding \cite{rissanen1983universal}
\begin{align}\label{eq:RissanenIntCode}
    \log^{*}m = \log c + \log m + \log \log m + \cdots\quad  (c\approx 2.865),
\end{align}
where the summation is only taken over nonnegative terms. 
The constant term only depends on the precision level $\delta$.
\end{prop}
The algorithm for the continuous case is presented in Algorithm~\ref{alg:main_continuous_case}. For computational efficiency, a practical algorithm can limit search range for $m_X$ and $m_Y$ to a subset $\Pcal \subset (\mathbb{N}^+)^2$.

\paragraph{Mixed Case: }
We can consider the both of mixed cases as a special case of the continuous case. If $X$ is a continuous variable and $Y$ is a discrete variable, we regard $y^n=(y_i)_i$ as continuous values by mapping $y^n$ to $\text{cont}(y^n; m_Y) = (y_i/m_Y)_i$ and calculate the description length using continuous causal models by $\Lcal^{c}(x^n, \text{cont}(y^n), m_X, m_Y; M)$. The causal discovery algorithm for the mixed case is presented in Algorithm~\ref{alg:main_mixed_case}.

\begin{algorithm}
\caption{Main function in Mixed Case in which $X$ is continuous}
\label{alg:main_mixed_case}
\begin{algorithmic}[1]
\Require Data \( z^n = (x^n , y^n) \), a set of model candidates \( \mathcal{M} \), search range $\Pcal\subset \mathbb{N}^+$ for $m_X$, and the number of categories $m_Y$
\Ensure Best model \( \hat{M} \)
\For{\( m_X\) in \( \Pcal \)}
    \For{\( M \) in \( \mathcal{M} \)}
        \State Compute \( \mathcal{L}^{c}(x^n, \text{cont}(y^n;m_Y), m_X, m_Y; M) \)  using \eqref{eq:Lc_MmXmY}
    \EndFor
\EndFor
\State \textbf{return} \( \hat{M} = \argmin_{M \in \Mcal} \min_{ m_X \in \Pcal }\mathcal{L}^{c}(x^n, \text{cont}(y^n; m_Y), m_X, m_Y; M) \)
\end{algorithmic}
\end{algorithm}

\section{Theoretical analysis of the statistical consistency}\label{sec5}
In this section,  we provide the theoretical analysis on the consistency of our method. By consistency,
we mean that the probability that our method select the true model converges to 1 at the limit of large $n$.
Noting the inclusion relation in our case, $M_{X \Perp Y}, M_{X \to Y}, M_{X \gets Y} \subset M_{X \gets C \to Y}$,
we consider the true model for a given probability distribution is the minimal model that contains it. 
\subsection{Discrete Case}
\begin{theorem}
Define true model $M^*(P^*)$ given as follows:
\begin{align*}
    M^*(P^*) = \begin{cases}
        M_{X \Perp Y} &  P^* \in M_{X \Perp Y} \\
        M_{X \to Y} &  P^* \in M_{X \to Y} \\
        M_{X \gets Y} &  P^* \in M_{X \gets Y} \\
        M_{X \gets C \to Y} &  P^* \in M_{X \gets C \to Y} \setminus \rbr{M_{X \Perp Y} \cup M_{X \gets Y} \cup M_{X \gets Y}  }.
    \end{cases}
\end{align*}
Then, the probability that \textsf{CLOUD} outputs $M^*(P^*)$ given $n$ i.i.d.\ samples from $P^*$, converges to 1 in the limit of $n\to \infty$, provided that the maximum likelihood estimation of $\fhat$ and $\ghat$ is successful. 
\end{theorem}
\begin{proof}
\paragraph{In case of $ P^* \in M_{X \Perp Y} $:}
The asymptotic expansion of log-likelihood~\cite{akaike1998information} implies
\begin{align}\label{eq:asymp-likelihood}
    -\log P\left(z^{n} ; M,\hat{\bm{\theta}}\left(z^{n}\right)\right)
    =   nH(P^*) + n \textnormal{KL}(P^*||P) + O_P\left(1\right),
\end{align}
for all $M\in\Mcal_{\rm all}$. Here, $H(P)$ denotes the entropy of $P$ defined as $$H(P)= -\sum_{k,k'} P(X=k,Y=k')\log P(X=k,Y=k'),$$ $\textnormal{KL}(P^*||P)$ denotes the KL-divergence defined as $$K(P^*||P)=\sum_{k,k'} P^{*}(X=k, Y=k') \frac{\log P^{*}(X=k, Y=k')}{\log P(X=k, Y=k')},$$ and $O_P(\cdot)$ denotes the asymptotic order with respect to $n$ in probability.
Specifically, in case of $M \in \cbr{M_{X \Perp Y}, M_{X \gets C \to Y}}$, the asymptotic expansion of Eq. \eqref{eq:asymp-likelihood} becomes
\begin{align}
    -\log P\left(z^{n} ; M,\hat{\bm{\theta}}\left(z^{n}\right)\right)
    =   nH(P^*) + O_P\left(1\right),
\end{align}
since $P^{*}$ belongs to both $M_{X \Perp Y}$ and $M_{X \gets C \to Y}$. We see $\log \Ccal_n(M_{X \Perp Y}) < \log \Ccal_n(M_{X \gets C \to Y})$, which leads to
\begin{align*}
    \Lcal^{d} (z^n;M_{X \gets C \to Y}) - \Lcal^{d} (z^n;M_{X \Perp Y}) 
    &= \log \Ccal_n(M_{X \gets C \to Y}) - \log \Ccal_n(M_{X \Perp Y}) + O_P\left( 1\right)
    \\
    &=\Omega(\log n)+O_P(1).
\end{align*}
Thus, the probability that $M_{X \gets C \to Y}$ achieves the smallest codelength converges to 0.
As for $M \in \lbrace M_{X \to Y}, M_{X \gets Y} \rbrace$, the negative log-likelihood function in the first term in Eq.~\eqref{stochastic_complexity} divided by $n$ converges to 
\begin{align}
\label{loglikelihood_LLN}
    - \frac{1}{n}\log P(z^n; \thetab) 
    &=- \sum_{k, k'} \frac{n(X=k, Y=k')}{n} \log P(X=k, Y=k'; \, {\thetab})  \notag \\
    &=- \sum_{k, k'} \theta_{k, k'}^* \log P(X=k, Y=k'; \,{\thetab}) +  o_P\left( 1\right)
\end{align}
as $n\to \infty$ since $\frac{n(X=k, Y=k')}{n} \to \theta_{k, k'}^{*}$. Since $P^*\notin M$ implies $\thetab \ne \thetab^{*}$ for those models, from the Gibbs inequality, this value is strictly larger than $H(P^*)$. The difference between the first terms gets dominant since the parametric complexity as well as the codelength to encode functions divided by $n$ converges to 0 for each model~\cite{kontkanen2008nml}. It then follows that the probability of having $\Lcal^{d} (z^n;M) > \Lcal^{d} (z^n;M_{X \Perp Y})$ tends to one, which implies the consistency of \textsf{CLOUD}.


\paragraph{In case of $ P^* \in M_{X \to Y}  $ or $ P^* \in M_{X \gets Y} $:}
By the symmetry, we restrict ourselves to the case of $ P^* \in M_{X \to Y} $ without loss of generality.
Let $\thetab^*$ be a parameter such that $P^*(X,Y) = P(X,Y;\thetab^*)$.
The first term of Eq.~\eqref{stochastic_complexity} for both $M_{X\to Y}$ and $M_{X \gets C \to Y}$ converges to $-\log P (z^n; \thetab^*)+O_P(1)$. As for the second term, we see
$\log \Ccal_n(M_{X \to Y}) < \log \Ccal_n(M_{X \gets C \to Y})$, which leads to
\begin{align*}
    \Lcal^{d} (z^n;M_{X \gets C \to Y}) - \Lcal^{d} (z^n;M_{X \to Y})
    &= \log \Ccal_n(M_{X \gets C \to Y})  - \log \Ccal_n(M_{X \to Y}) + O_P\left( 1\right)
    \\
    &=\Omega(\log n)+O_P(1).
\end{align*}
Therefore, the probability that $M_{X\gets C\to Y}$ achieves the smallest codelength converges to 0. 
As discussed in the case above, the negative log-likelihood function in the first term in Eq.~\eqref{stochastic_complexity} divided by $n$ converges to strictly larger value than $H(P^*)$ if $M\in\cbr{M_{X \Perp Y}, M_{X \gets Y}}$. The difference between the first terms gets dominant as mentioned above. 
Hence, the probablity that $\Lcal^{d}(z^n; M_{X\to Y})$ is shortest converges to 1.

\paragraph{In case of $P^* \in M_{X \gets C \to Y} \setminus \rbr{M_{X \gets Y} \cup M_{X \gets Y} \cup \cup M_{X \Perp Y} }$:}
As discussed in the case above, the negative log-likelihood function in the first term in Eq.~\eqref{stochastic_complexity} divided by $n$ converges to strictly larger value than $H(P^*)$ if $M\ne M_{X\gets C\to Y}$. Since the first term is dominant as mentioned in above, the probability that $M_{X \gets C \to Y}$ will be selected converges to 1 as $n\to\infty$.
\end{proof}

\subsection{Continuous Case}
\begin{theorem}
Let $\Pcal$ be a finite set so that $\cup_{(m_X,m_Y)\in\Pcal} \textsf{HIS}^{m_X,m_Y} = \textsf{HIS}^{\mbar_X,\mbar_Y}$ holds for some $(\mbar_X, \mbar_Y)\in \Pcal$.
For true distribution $p^*$, we define $p^*_{m_X,m_Y} \in \textsf{HIS}^{m_X,m_Y}$ as follows: 
$$ p^*_{m_X,m_Y}(x,y) = \frac{\iint_{(x',y') \in I(x,y)} p^*(x',y')dx'dy'}{m_Xm_Y} , $$
where $I(x,y) = I_k^X\times I_{k'}^Y$ such that $x\in I_k^X$ and $y\in I_{k'}^Y$ holds.
Then we define true model $M^*(p^*)$ as follows:
\begin{align*}
    M^*(p^*) = \begin{cases}
        M_{X \Perp Y} &  p^*_{\mbar_X,\mbar_Y} \in M_{X \Perp Y} \\
        M_{X \to Y} &  p^*_{\mbar_X,\mbar_Y} \in M_{X \to Y} \\
        M_{X \gets Y} &  p^*_{\mbar_X,\mbar_Y} \in M_{X \gets Y} \\
        M_{X \gets C \to Y} &  p^*_{\mbar_X,\mbar_Y} \in M_{X \gets C \to Y} \setminus \rbr{M_{X \Perp Y} \cup M_{X \gets Y} \cup M_{X \gets Y}  } .
    \end{cases}
\end{align*}
Then, the probability that \textsf{CLOUD} outputs $M^*(p^*)$ using $\Pcal$ and $n$ i.i.d.\ samples from $p^*$ converges to 1 in the limit of $n\to \infty$, provided that the maximum likelihood estimation of $\fhat$ and $\ghat$ is successful. 
\end{theorem}

\begin{proof}
From the definition of $\Lcal^{c}(z^n,m_X,m_Y; M)$ in Eq.~\eqref{eq:Lc_MmXmY}, we see
\begin{align}\label{eq:Lc_n_order}
&\Lcal^{c}(z^n,m_X,m_Y; M)\notag\\
&=
 \Lcal^{d}(\textnormal{disc}(x^n; m_X) , \textnormal{disc}(y^n; m_Y); \textnormal{DISC}(M; m_X, m_Y)) -n\log m_Xm_Y +\textnormal{const.} \notag\\
 &= \min_{P\in \textrm{DISC}(M;m_X,m_Y)} -\log P(\textnormal{disc}(x^n;m_X),\textnormal{disc}(y^n;m_Y)) + 
\Ccal_n(\textrm{DISC}(M;m_X,m_Y) ) \notag \\
&\phantom{\min_{P\in \textrm{DISC}(M;m_X,m_Y)}x}-n \log m_Xm_Y + \textnormal{const.}
\end{align}
For fixed $\Phat \in \textnormal{DISC}(M,m_X,m_Y)$, there is a corresponding density function $\phat\in M^{m_X,m_Y}\subset \textsf{HIS}^{m_X,m_Y}$ such that $m_Xm_Y\Phat(\textnormal{disc}(x^n;m_X),\textnormal{disc}(y^n;m_Y)) = \phat(x^n,y^n)$ for all $x^n$ and $y^n$. Using this 
correspondence, we see the following expansion in probability:
\begin{align}
\label{eq:aymptotic_likelihood_expansion_cont}
-\log \Phat(\textnormal{disc}(x^n;m_X),\textnormal{disc}(y^n;m_Y)) - n\log m_Xm_Y
&= -\log \hat p(z^n) \notag \\
&= nH(p^*) + n\textnormal{KL}(p^*\| \phat) + O_P(1),
\end{align}
where $H(p) = \iint p(x,y)\log\frac 1{p(x,y)}dxdy$
and $\mathrm{KL}(p||p') = \iint p(x,y)(\log {p(x,y)}-\log {p'(x,y)})dxdy$. We thus obtain
\begin{align*}
\Lcal^{c}(z^n,m_X,m_Y; M ) 
&= nH(p^*)+n\min_{\phat \in M^{m_X,m_Y}} \mathrm{KL}(p^*\mathop{\|}\phat) \\
& + \Ccal_n(\mathrm{DISC}(M; m_X,m_Y)) + O_P(1).  
\end{align*}
From the assumption on $\Pcal$, we further obtain
\begin{align*}
&\min_{(m_X,m_Y)\in\Pcal} \Lcal^{c}(z^n,m_X,m_Y; M ) \\
&= nH(p^*) + n\min_{(m_X,m_Y)\in\Pcal}\min_{\phat \in M^{m_X,m_Y}}\mathrm{KL}(p^*\mathop{\|}\phat) 
+ \Ccal_n(\mathrm{DISC}(M;m_X,m_Y)) + O_P(1)\\
&= nH(p^*) + n\min_{\phat \in M^{\mbar_X,\mbar_Y}}\mathrm{KL}(p^*\mathop{\|}\phat) 
+ \Ccal_n(\mathrm{DISC}(M;\mbar_X,\mbar_Y)) + O_P(1).
\end{align*}
For any $\phat \in \mathsf{HIS}^{m_X,m_Y}$, we see
\begin{align*}
&\mathrm{KL}(p^*\mathop{\|}\phat) \\
&= \sum_{k,k'}\iint_{(x,y)\in I^X_k\times I^Y_{k'}} p^*(x,y)\rbr{\log \frac{p^*(x,y)}{ p^*_{m_X,m_Y}(x,y)}+\log \frac{p^*_{m_X,m_Y}(x,y)}{\phat(x,y)}}dxdy \\
&= \mathrm{KL}(p^*\mathop{\|}p^*_{m_X,m_Y}) + \sum_{k,k'}\iint_{(x,y)\in I^X_k\times I^Y_{k'}}  p^*(x,y)\rbr{\log \frac{p^*_{m_X,m_Y}(x,y)}{\phat(x,y)}}dxdy\\
&= \mathrm{KL}(p^*\mathop{\|}p^*_{m_X,m_Y}) + \sum_{k,k'}\iint_{(x,y)\in I^X_k\times I^Y_{k'}} p^*_{m_X,m_Y}(x,y)\rbr{\log \frac{p^*_{m_X,m_Y}(x,y)}{\phat(x,y)}}dxdy\\
&= \mathrm{KL}(p^*\mathop{\|}p^*_{m_X,m_Y}) + \mathrm{KL}(p^*_{m_X,m_Y}\mathop{\|}\phat).
\end{align*}
The third equality holds since $\rbr{\log p^*_{m_X,m_Y}(x,y)-\log\phat(x,y)}$ is constant for all $(x,y)\in I^X_k\times I^Y_{k'}$ and $\iint_{(x,y)\in I^X_k\times I^Y_{k'}} p^*(x,y) dxdy= \iint_{(x,y)\in I^X_k\times I^Y_{k'}} p^*_{m_X,m_Y}(x,y)dxdy$.
Therefore, $\phat^*_{m_X,m_Y}$ is a unique minimizer of 
$\mathrm{KL}(p^*\mathop{\|} \phat)$, which implies that
\begin{align*}
\min_{(m_X,m_Y)\in\Pcal} \Lcal^{c}(z^n,m_X,m_Y; M ) 
&= nH(p^*)   + n \mathrm{KL}(p^*\mathop{\|}p^*_{\mbar_X,\mbar_Y}) \\
&+ \Ccal_n(\mathrm{DISC}(M;\mbar_X,\mbar_Y)) + O_P(1). 
\end{align*}
for $M$ such that $p^*_{\mbar_X,\mbar_Y}\in M$. If $p^*_{\mbar_X,\mbar_Y}\notin M$ otherwise, we see
\begin{align*}
 &\min_{(m_X,m_Y)\in\Pcal} \Lcal^{c}(z^n,m_X,m_Y; M ) -\min_{(m_X,m_Y)\in\Pcal} \Lcal^{c}(z^n,m_X,m_Y; M_{X\gets C\to Y} )\\
 &= \min_{(m_X,m_Y)\in\Pcal}\Lcal^{c}(z^n,m_X,m_Y; M )-nH(p^*) - n \mathrm{KL}(p^*\mathop{\|}p^*_{\mbar_X,\mbar_Y}) +o_P(n)\\
 &= \min_{\phat \in M^{\mbar_X,\mbar_Y}}   n  \mathrm{KL}(p^*_{\mbar_X,\mbar_Y}\mathop{\|}\phat) +o_P(n),
\end{align*}
since it always holds $p^*_{\mbar_X,\mbar_Y}\in  M_{X\gets C\to Y}$. This implies that the probability that $M$ is chosen converges to 0.
The consistency among models in which $p^*_{\mbar_X,\mbar_Y}\in M$ follows similarly as that of \textsf{CLOUD} in the discrete case.
\end{proof}

\section{Experiment}\label{sec6}
In this section, we demonstrate that 1) our proposed method \textsf{CLOUD} can solve the Reichenbach problem in situations where it is difficult to make assumptions about unobserved common causes, and 2) \textsf{CLOUD} shows higher inference accuracy in identifying causal relationships compared to existing methods, even when the true data-generating mechanism violates the assumptions of our method.

In the first experiment, we designed the Reichenbach problem scenarios with all set of synthetic data types — discrete, mixed, and continuous, and verified that our proposed method exhibits high accuracy and consistency in solving the problem. In particular, it effectively detects the presence of unobserved common causes $C$ even in situations where the observed variables $X, Y$ are generated from complex mechanisms $f(X,C), g(Y,C)$.

In the second experiment, we compared the performance of our proposed method against existing methods in identifying causal relationships in synthetic data generated from either  $M_{X \to Y}$ or $M_{X \gets C \to Y}$, and demonstrated its effectiveness.

The third experiment tested the ability of our proposed method to determine the directions of causality and detect unobserved common causes in real-world data generated from unknown and complex data-generating process.

We implemented \textsf{CLOUD} in Python and provide the source code at \url{https://github.com/Matsushima-lab/CLOUD}.

\subsection{Consistency of \textsf{CLOUD} on the Reichenbach problem}\label{subsec6-1}
In the first experiment, we verified the performance of \textsf{CLOUD} on the Reichenbach problem with synthetic data and confirmed its consistency in model selection. We randomly selected each SCM corresponding to the four causal relationships, and then generated data $z^n$ with sample size $n=10^2, 10^3, 10^4$. The data-generating processes for each combination were defined as follows:
\paragraph{Discrete Case:}
\begin{outline}
        \1 $M_{X \Perp Y}$: \\
        $X$ and $Y$ were independently generated from categorical distributions.
        \1 $M_{X \gets C \to Y}$: \\
        $C \in \cbr{0,1,\ldots, 99}$ was independently generated from a categorical distribution, and $X$ and $Y$ were set to the quotient and remainder of $C$ divided by 10, respectively.
        \1 $M_{X \to Y}$: \\
        $X$ and $ E_Y$ were independently generated from categorical distributions, $f$ was generated uniformly randomly from all non-constant functions, and subsequently $Y$ was set to $Y = f(X) + E_Y (\bmod\  10)$. The same applied to the case of $M_{X\gets Y}$.
\end{outline}
\paragraph{Mixed Case ($X$ is continuous and $Y$ is discrete):}
\begin{outline}    
        \1 $M_{X \Perp Y}$: \\
        $X$ was independently generated from a Gaussian distribution, and $Y$ from a categorical distribution.
        \1 $M_{X \gets C \to Y}$: \\
        $C \in \cbr{0,1,\ldots, 99}$ was independently generated from a categorical distribution. $X$ and $Y$ were then generated as follows: $X = b \sin{C} + E_X, Y = \lfloor \frac{C}{10} \rfloor$, where $b$ was sampled from a uniform distribution $\mathcal{U}(2, 4)$ and $E_X$ follows a Gaussian distribution $\mathcal{N}(0, 0.1^2)$.
        \1 $M_{X \to Y}$: \\
        $X$ was generated from a mixture of Gaussian distributions with three clusters as:
        \[p(X) = 0.6 \cdot \mathcal{N}(X ; -5, 2^2) + 0.2 \cdot \mathcal{N}(X ; 0, 1^2) + 0.2 \cdot \mathcal{N}(X ; 5, 2^2).\]     Then, $X$ was divided into $m_X \sim \text{Uniform}\{2, 3, 4\}$ equal intervals, with $f(X) \sim \text{Uniform}\{0, 1, \ldots, 10\}$ assigned to each interval. Additive noise $E_Y \sim \text{Uniform}\{-1, 0, 1\}$ was added to generate $Y$, with addition over $\mathbb{R}/10\mathbb{Z}$. The correlation coefficient was ensured to be greater than 0.2.
        \1 $M_{X \gets Y}$: \\
        $Y$ was generated from an $m_Y$-valued categorical distribution, and then $X$ was set to $X = 2Y+3\sin{Y}+E_Y$, where $m_Y$ was generated from uniform distribution $\text{Uniform}\{2, 3, \ldots, 8\}$ and $E_Y \sim \mathcal{N}(0, 1)$, with addition taken over $\mathbb{R}/20\mathbb{Z}$.
\end{outline}
\paragraph{Continuous  Case:}
\begin{outline}    
        \1 $M_{X \Perp Y}$: \\
        $X$ and $Y$ were independently generated from Gaussian distributions.
        \1 $M_{X \gets C \to Y}$: \\
        $X$ and $Y$ were generated based on an ellipse equation with eccentricity $e$ and semi-major axis $a$: $r = \frac{a(1 - e^2)}{1 + e\cos(\phi)}$ with $0 \leq \phi < 2\pi$ In this process, we sampled $e$ from $\mathcal{U}(0.5, 0.9)$ and $a$ from $\text{Uniform}\{1, 2, 3\}$. $X$ and $Y$ were then set to $X = r\cos(\phi+\eta) + E_X, Y = r\sin(\phi+\eta) + E_Y (0 \leq \phi \leq 2 \pi)$ , where $\eta$ were sampled from $\mathcal{U}(\pi/4, \pi/3) $, and $ E_X$ and $E_Y$ follow $ \mathcal{N}(0, (0.1a)^2)$. 
        \1 $M_{X \to Y}$: \\
        $X$ was generated from a probability density function of a mixture Gaussian distribution with three clusters. Then $Y$ was set to $Y = a * \text{disc}(X, m_X) + b + E_Y$, where $m_X \sim \text{Uniform}\{2, 3, 4\}, a \sim \mathcal{U}(4, 7), b \sim \mathcal{U}(1, 5), E_Y \sim \mathcal{N}(0, 1)$, with addition taken over $\mathbb{R}/20\mathbb{Z}$. The procedure for the $M_{X\gets Y}$  was analogous.

\end{outline}

\begin{table}[ht]
\centering
\caption{Results on experiment 1\label{exp1_table}}
\let\cline\clineorig
\renewcommand{\arraystretch}{1.4}
\begin{tabular}{|l|c|c|c|c|c|c|c|c|c|c|c|c|}
\hline
\multirow{2}{*}{$n$} & \multicolumn{3}{c|}{\(M_{X \Perp Y}\)} & \multicolumn{3}{c|}{\(M_{X \gets C \to Y}\)} & \multicolumn{3}{c|}{\(M_{X \to Y}\)} & \multicolumn{3}{c|}{\(M_{X \gets Y}\)}\\
\cline{2-13}
& disc. & mix. & cont. & disc. & mix. & cont. & disc. & mix. & cont. & disc. & mix. & cont. \\
\hline
$10^2$ & 95.1 & 91.9 & 90.4 & 85.0 & 96.1 & 95.2 & 24.6 & 91.4 & 74.8 & 26.8 & 77.4 & 75.4 \\
$10^3$ & \textbf{100} & 99.8 & 96.3 & 87.6 & \textbf{100} & \textbf{100} & 99.6 & \textbf{100} & 98.4 & 99.9 & 99.4 & 98.9\\
$10^4$ & \textbf{100} & \textbf{100} & \textbf{100} & \textbf{100} & \textbf{100} & \textbf{100} & \textbf{100} & \textbf{100} & 99.7 & \textbf{100} & \textbf{100} & 99.9 \\
\hline
\end{tabular}
\end{table}

Table~\ref{exp1_table} shows the transition of accuracy as a fraction of correct inference with sample size $n$. Figures \ref{fig:exp1_discrete_cm}, \ref{fig:exp1_mixed_cm}, and \ref{fig:exp1_continuous_cm} visualize the inference results of \textsf{CLOUD} as confusion matrices for each data type. These indicate that accuracy improves as $n$ increases. In all cases, the accuracy reaches $\sim$100\% at $n=10000$.
 We thus empirically observed the consistency of \textsf{CLOUD}.

\begin{figure}[ht]
 \centering
 \includegraphics[width=\linewidth]{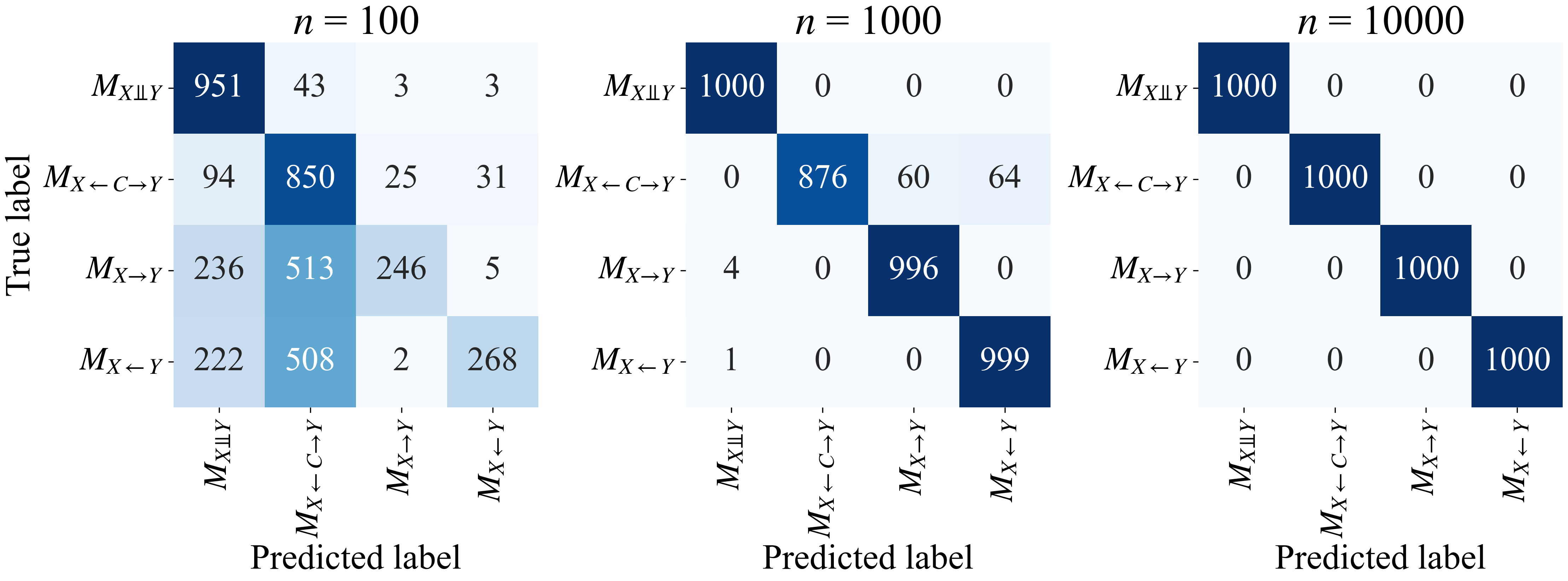}
 \caption{Confusion matrices in the Discrete Case of Experiment 1}
 \label{fig:exp1_discrete_cm}
\end{figure}

\begin{figure}[ht]
 \centering
 \includegraphics[width=\linewidth]{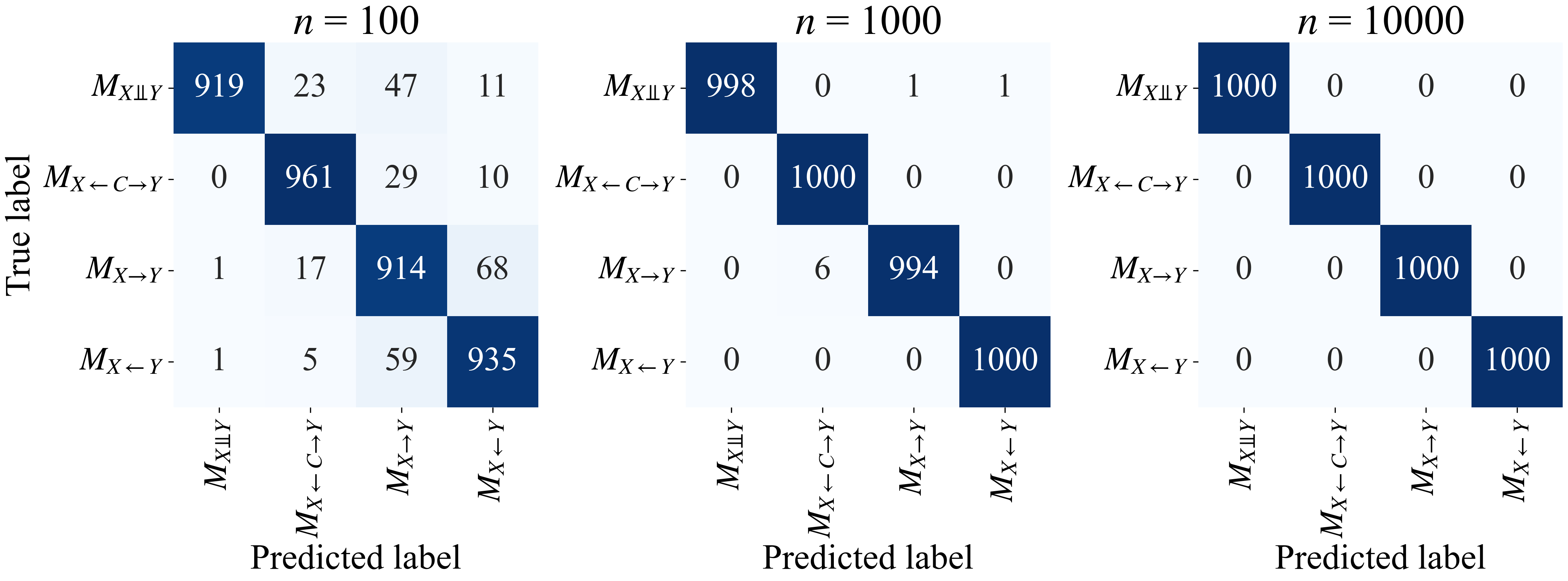}
 \caption{Confusion matrices in the Mixed Case of Experiment 1}
 \label{fig:exp1_mixed_cm}
\end{figure}

\begin{figure}[ht]
 \centering
 \includegraphics[width=\linewidth]{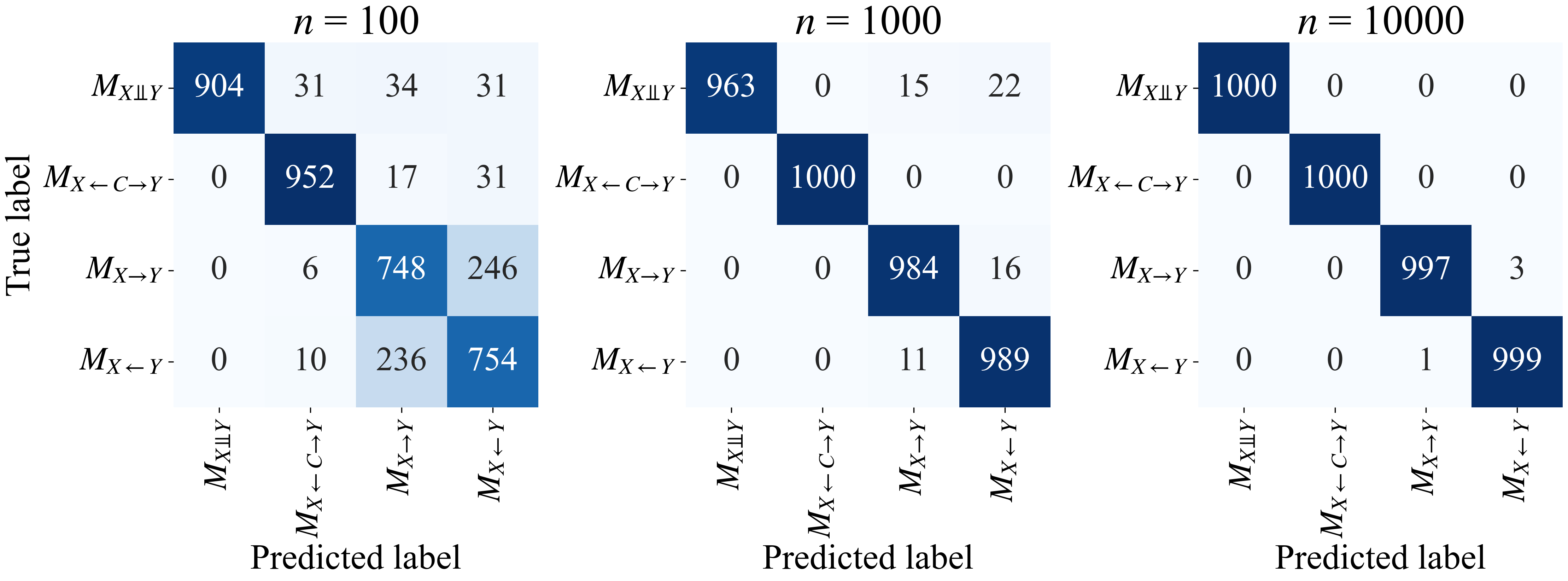}
 \caption{Confusion matrices in the Continuous Case of Experiment 1}
 \label{fig:exp1_continuous_cm}
\end{figure}

\textsf{CLOUD} calculates the codelength of the observed data for each causal model and selects the one that achieves the shortest codelength.
Therefore, we can expect that our method is more confident in its inference when the difference in the codelengths between the shortest one and the rest is larger.
The next experiment examined whether the difference in the codelengths per sample size of the shortest and the second shortest model, denoted as $\Delta$, can be interpreted as the confidence of \textsf{CLOUD}. We generated 1000 synthetic datasets from discrete causal models and calculated the accuracy at each decision rate $d$.
The accuracy at decision rate $d$ is defined as the accuracy at the upper $d$\% of datasets when datasets are sorted in descending order of $\Delta$. 
The result is shown in Fig~\ref{fig:dicision_rate}. For each model, the accuracy is higher when the decision rate is smaller, i.e., $\Delta$s are larger.
We thus conclude that $\Delta$ is interpreted as the confidence of the inference in \textsf{CLOUD}.

\begin{figure}[ht]
 \centering
 \includegraphics[width=0.6\linewidth]{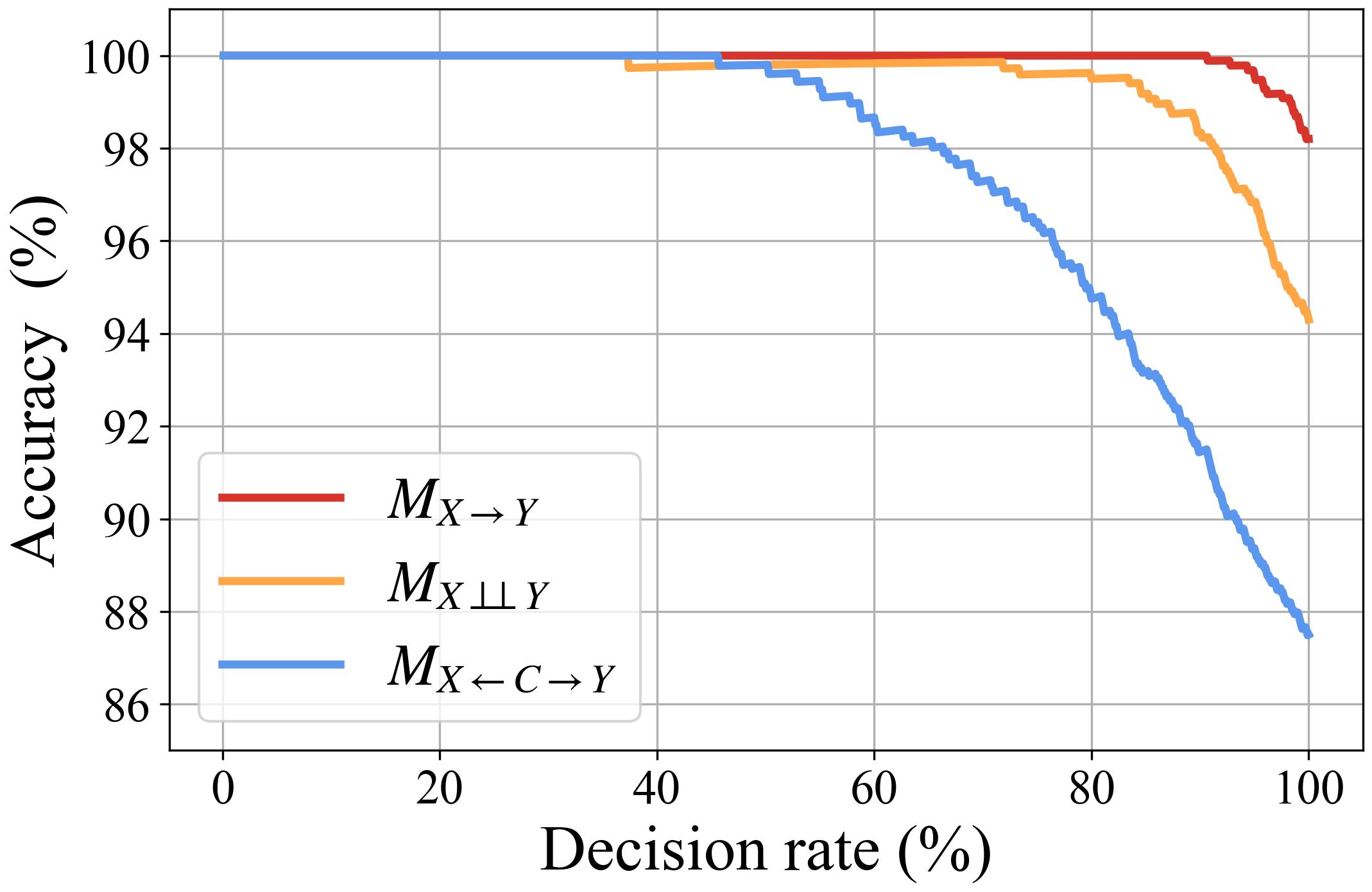}
 \caption{Accuracy vs. decision rate of \textsf{CLOUD} on synthetic data
 \label{fig:dicision_rate}}
\end{figure}

\subsection{Comparison to existing methods in case of $X \to Y$ and $X \gets C \to Y$}\label{subsec6-2}
In the second experiment, we compared the accuracy of out proposed method \textsf{CLOUD} with existing methods in two scenarios where the ground truth of the causal relationship is either $X \to Y$ or $X \gets C \to Y$. This comparison aimed to evaluate \textsf{CLOUD}'s performance in identifying the direction of causality and in detecting unobserved common causes. We set the sample size at $n=500$. For each SCM, synthetic datasets $z^{n=500}=x^{n=500} \times y^{n=500}$ were generated 100 times. The accuracy was determined by calculating the proportion of correctly identified causal relationships across these iterations.

To verify the cases where the true data-generating process violates the assumptions of our models, for $M_{X\to Y}$, we employ a non-cyclic ANM in which addition is taken without modulo operation:
\begin{outline}
    \1 Discrete Case \\
    $X $ and $E_Y$ were randomly and independently generated from categorical distributions, and $Y$ was set to $Y=f(X)+E_Y$ with a mapping function $f$ that was also randomly set.
    \1 Mixed Case [$X$ is continuous and $Y$ is discrete] \\
    $X$ was generated from $\mathcal{N}(0, 10^2)$, and then $X$ was divided into $m_X \sim \text{Uniform}\{2, 3, 4\}$ equal intervals. For each interval, a value $f(X) \sim \text{Uniform}\{1, 2, \ldots, 24\}$ was randomly assigned, and additive noise $E_Y \sim \text{Uniform}\{-1, 0, 1\}$ was added to generate $Y$.
    \1 Continuous Case \\
    $X$ was generated from a three-class mixture Gaussian distribution as in Experiment 1, and $Y$ was generated from $Y=af(X)+b\sin{2 \pi X}+E_Y$, where $a, b \sim \mathcal{U}([-2, -0.5] \cup [0.5, 2])$, and $E_Y \sim \mathcal{N}(0, 1)$. For the function $f(X)$, we considered two cases: one linear and the other a cubic function $x^3$.
\end{outline}

In the case where $M_{X \gets C \to Y}$ is the ground truth, the same process used in Experiment 1 was used.

For existing methods, we employed \textsf{ECI} \cite{kocaoglu2017entropic}, \textsf{DR} \cite{peters2011causal}, \textsf{CISC} \cite{budhathoki2017mdl}, \textsf{ACID} \cite{budhathoki2018accurate}, \textsf{LiM} \cite{zeng2022causal}, \textsf{HCM} \cite{li2022hybrid}, \textsf{LiNGAM} \cite{shimizu2011directlingam}, \textsf{ANM} \cite{peters2014causal}, \textsf{RCD} \cite{maeda2020rcd}, \textsf{CAMUV} \cite{maeda2021causal}, \textsf{BUPL} \cite{tashiro2014parcelingam}, and \textsf{COCA} \cite{kaltenpoth2019we}. These methods are categorized into those that can detect $M_{X \gets C \to Y}$ and those that cannot. While \textsf{ECI}, \textsf{DR}, \textsf{CISC}, \textsf{ACID}, \textsf{LiM}, \textsf{HCM}, \textsf{LiNGAM}, and \textsf{ANM} distinguish between $M_{X \to Y}$ and $M_{X \gets Y}$, \textsf{RCD}, \textsf{CAMUV}, and \textsf{BUPL} infer a causal model from a set of model candidates including $M_{X \gets C \to Y}$. Moreover, \textsf{COCA} selects a model only from $\Mcal = \lbrace M_{X \to Y}, M_{X \gets C \to Y} \rbrace$. We note that our proposed method \textsf{CLOUD} as well as \textsf{LiM} and \textsf{HCM} can accept all type of data, wheres others specialize for either discrete or continuous data.

We utilized the implementations of \textsf{HCM} by Li et al. (2022), \textsf{COCA} by Kaltenpoth et al. (2019), and others by Ikeuchi et al. (2023). Default hyper parameter values were used.

Results are shown in Table \ref{ExpComparizon}. Unlike existing methods, \textsf{CLOUD} is applicable to all experimental conditions. In particular, \textsf{CLOUD} is the first method capable of detecting unobserved common causes in discrete and mixed cases. As Table \ref{ExpComparizon} demonstrates, \textsf{CLOUD} showed consistently high inference accuracy across all cases, regardless of the data type, even though the number of model candidates of \textsf{CLOUD} is as many as 4 models. Especially, \textsf{CLOUD} outperformed previous methods in the discrete case.

\begin{table}[ht]
\centering
\caption{Results on experiment 2
\label{ExpComparizon}}
\let\cline\clineorig
\renewcommand{\arraystretch}{1.4}
\begin{tabular}{|l|c|c|c|c|c|c|c|c|c|c|c|c|}
\hline
\multirow{2}{*}{Methods} & \multicolumn{4}{c|}{Direct Case} & \multicolumn{3}{c|}{Confounded Case} & \multirow{2}{*}{\(|\mathcal{M}|\)}  & \multirow{2}{*}{\(C\)}\\ \cline{2-8}
& disc. & mix. & \multicolumn{2}{c|}{cont.} & disc. & mix. & cont. &  & \\ \cline{4-5}
& & & linear & cubic & & & & & \\
\hline
\textsf{ECI} & 89\% & - & - & - & - & - & - & 2 & - \\
\textsf{DR} & 85\% & - & - & - & - & - & - & 2 & - \\
\textsf{CISC} & 96\% & - & - & - & - & - & - & 2 & - \\
\textsf{ACID} & 92\% & - & - & - & - & - & - & 2 & - \\
\textsf{LiM} & 86\% & 12\% & 82\% & 87\% & - & - & - & 3 & - \\
\textsf{HCM} & 90\% & \textbf{100\%} & 85\% & \textbf{100\%} & - & - & - & 3 & - \\
\textsf{LiNGAM} & - & - & 95\% & 61\% & - & - & - & 3 & -  \\
\textsf{ANM} & - & - & 91\% & \textbf{100\%} & - & - & - & 2 & - \\ \hline
\textsf{RCD} & - & - & 48\% & 5\% & - & - & 96\% & 4 & \(\ast\) \\ 
\textsf{CAMUV} & - & - & 92\% & 99\% & - & - & \textbf{100\%} & 4 & \(\ast\) \\
\textsf{BUPL} & - & - & 37\% & 8\%& - & - & \textbf{100\%} & 4 & \(\ast\)\\
\textsf{COCA} & - & - & 2\% & 34\% & - & - & 13\% & 2 & \(\ast\) \\
\textbf{\textsf{CLOUD}}& \textbf{98\%} & 99\% & \textbf{96\%} & 99\% & \textbf{88\%} & \textbf{100\%} & \textbf{100\%} & 4 & \(\ast\) \\
\hline
\end{tabular}
\vspace{.5ex}
{\raggedright Performance comparison of \textsf{CLOUD} against existing methods w.r.t. accuracy in the discrete case (disc.), mixed case (mix.) and continuous case (cont.) based on synthetic data generated either from direct case or confounded case. $| \mathcal{M} | $ column denotes the number of model candidates each method considers, and $C$ one represents whether each method allows for the existence of unobserved common causes or not ($\ast$: Yes, -: No) \par}
\end{table}

\subsection{Real World Data}\label{subsec6-3}

\subsubsection{Direct case: T\"{u}bingen Benchmark Pairs}\label{subsubsecC3a}

In this section, we examined the effectiveness of \textsf{CLOUD} of inferring direct causality in data generated from complex causal mechanisms by real-world datasets.

We employed datasets in various application fields from the Tübingen Cause-Effect Pairs Database \cite{mooij2016distinguishing}, which provides a collection of datasets for testing causal discovery methods. The database contains datasets with known ground truth to distinguish between cause and effect variables.
We note that the ground truth does not mean there are no unobserved counfounders in general, except for dataset No.101 which was explicitly generated in an unconfounded experimental environment. The statistical information of the datasets used in the experiments is shown in Table \ref{tab:stasTubingen}, and scatter plots for each data pair are presented in Figure \ref{fig:scatterTubingen}. Descriptions for each data pair are given in Appendix~\ref{secD1}.

\begin{table}[h]
\centering
\caption{Characteristics of Tübingen Cause-Effect-Pairs\label{tab:stasTubingen}}
\let\cline\clineorig
\renewcommand{\arraystretch}{1.5} 
\addtolength{\tabcolsep}{-1pt}
\begin{tabular}{lcccll}
\toprule
Dataset  & Data-type & $n$ & Ground Truth & $X$ & $Y$ \\ \midrule
No. 47   & disc. & 254 & $X \gets Y$ & \text{\scriptsize number of cars}  & \text{\scriptsize working days or not} \\ \hline
No. 68   & disc. & 498 & $X \gets Y$ & \text{\scriptsize bytes sent at minute}  & \text{\scriptsize open http connections} \\ \hline
No. 107   & disc. & 240 & $X \to Y$ & \text{contrast}  & \text{\scriptsize answer correct or not} \\ \hline
No. 85   & mix. & 994 & $X \to Y$ & \text{day}  & \text{\tiny protein content of the milk} \\ \hline
No. 95   & mix. & 9504 & $X \to Y$ & \text{\scriptsize hour of the day}  & \text{\tiny total electricity consumption} \\ \hline
No. 99   & mix. & 2287 & $X \gets Y$ & \text{\tiny language test score}  & \text{\scriptsize social-economic status} \\ \hline
No. 23   & cont. & 452 & $X \to Y$ & \text{age}  & \text{weight} \\ \hline
No. 77   & cont. & 8401 & $X \gets Y$ & \text{\tiny daily average temperature}  & \text{\scriptsize solar radiation} \\ \hline
No. 101   & cont. & 300 & $X \to Y$ & \text{\scriptsize grey value of a pixel}  & \text{\scriptsize light intensity} \\ \hline
\bottomrule
\end{tabular}
\end{table}

\begin{figure}[h]
 \centering
 \includegraphics[width=0.9\linewidth]{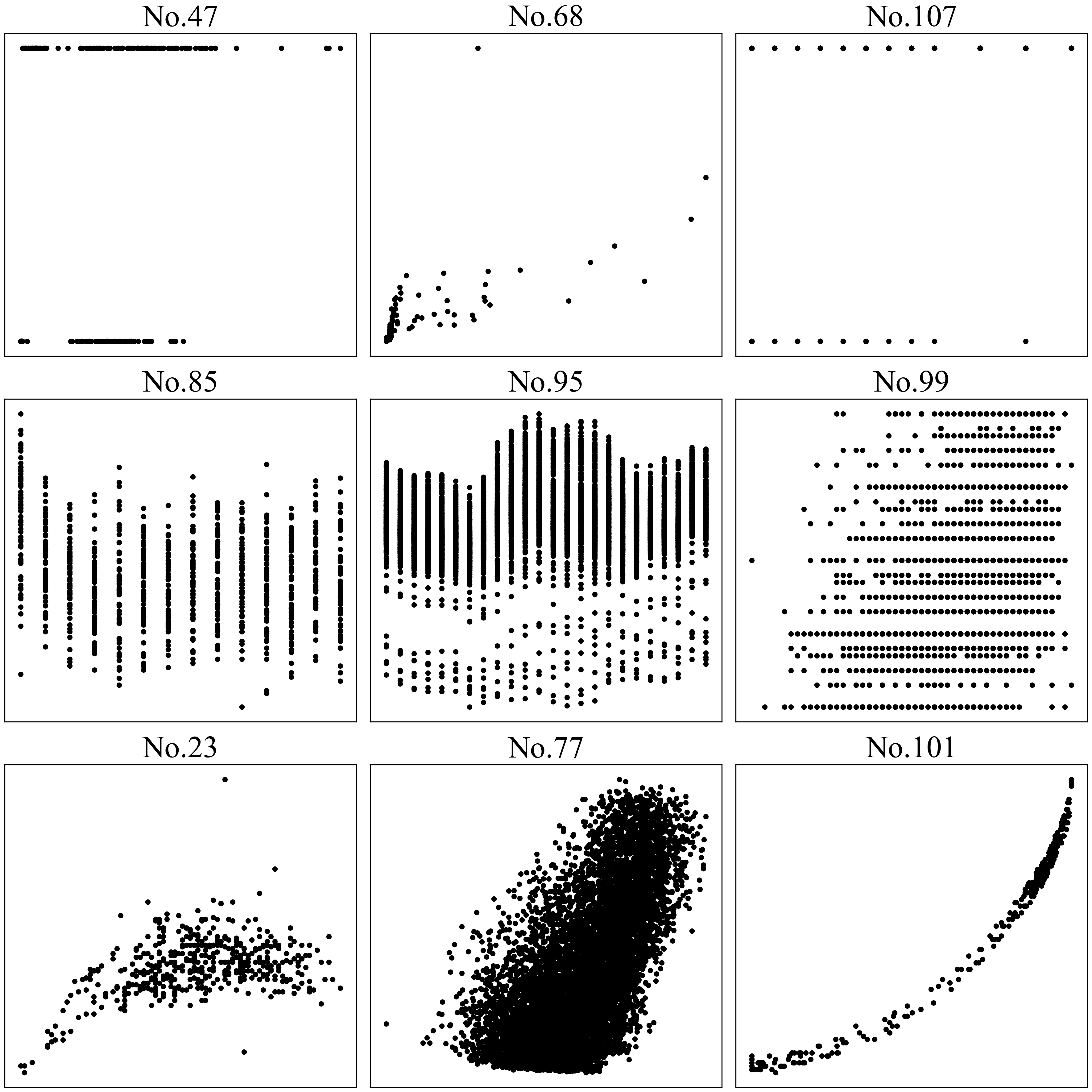}
 \caption{Scatter plots of the Tübingen Cause-Effect Pairs. The horizontal axis represents $X$, while the vertical axis represents $Y$. Each plot corresponds to a dataset pair.
 \label{fig:scatterTubingen}}
\end{figure}

We determined the data type for each data pair based on the information of the variables and run applicable causal discovery methods. Results are shown in Table \ref{tab:resultTugingen}. \textsf{CLOUD} demonstrated an excellent ability to determine the causal directions across various data types. Notable, it correctly identified the causal directions in every case both in mixed and continuous cases. While $\textsf{LiM}$ and $\textsf{HCM}$ are applicable across all cases, \textsf{CLOUD} achieved the highest number of correct answers.

In discrete case, \textsf{CLOUD} showed performance comparable to \textsf{CISC}, a state-of-the-art causal discovery method for discrete data. \textsf{CLOUD} also correctly inferred the right directions with high confidence (large $\Delta$) for cases No.47 and No.68, where \textsf{LiM} and \textsf{HCM} were incorrect. Notably, in case No.107, \textsf{CLOUD} detected the presence of a confounding factor (indicated as ‘conf’) rather than a clear causal direction. This is a crucial feature in real-world data analysis involving potential confounding variables. 

In continuous case, \textsf{CLOUD} successfully identified the causal directions across all pairs, despite having four model candidates. In contrast, \textsf{RCD}, \textsf{CAMUV}, and \textsf{BUPL}, which also have $|\Mcal|=4$, frequently resulted in the ‘conf’ (confounded) classification. This tendency suggests a bias in these methods towards indicating confoundedness rather than directly identifying causality. Particularly, it would be incorrect to conclude the presence of confounders in dataset No.101 mentioned above.

Overall, the results from the T\"{u}bingen Benchmark Pairs suggests that $\textsf{CLOUD}$ is a reliable causal discovery method to identify the true causal direction, particularly in the situations involving complex data-generating process across all data types.

\begin{table}[h]
\centering
\caption{Results on T\"{u}bingen Benchmark Pairs Dataset
\label{tab:resultTugingen}}
\let\cline\clineorig
\renewcommand{\arraystretch}{1.4}
\addtolength{\tabcolsep}{-5pt}
\begin{tabular}{|p{1.2cm}|p{1.2cm}<{\centering}|p{1.2cm}<{\centering}|p{1.2cm}<{\centering}|p{1.2cm}<{\centering}|p{1.2cm}<{\centering}|p{1.2cm}<{\centering}|p{1.2cm}<{\centering}|p{1.2cm}<{\centering}|p{1.2cm}<{\centering}|p{1.2cm}<{\centering}|p{1.2cm}<{\centering}|p{1.2cm}<{\centering}|}
\hline
\multirow{2}{*}{Methods} & \multicolumn{3}{c|}{Discrete Case} & \multicolumn{3}{c|}{Mixed Case} & \multicolumn{3}{c|}{Continuous Case}\\
\cline{2-10}
& No.47 & No.68 & No.107 & No.85 & No.95 & No.99 & No.23 & No.77 & No.101 \\
\hline
\textsf{ECI} & $\checkmark$ & $\times$ & $\times$ & - & - & - & - & - & - \\
\textsf{DR} & $\thickapprox$ & $\thickapprox$ & $\thickapprox$ & - & - & - & - & - & -\\
\textsf{CISC} & $\checkmark$ & $\checkmark$ & $\times$ & - & - & - & - & - & - \\
\textsf{ACID} & $\times$ & $\times$ & $\thickapprox$ & - & - & - & - & - & -\\
\textsf{LiM} & $\times$ & $\times$ & $\times$ & $\checkmark$ & $\checkmark$ & $\checkmark$ & $\times$ & $\times$ & $\times$\\
\textsf{HCM} & $\times$ & $\times$ & $\checkmark$ & $\checkmark$ & $\checkmark$ & $\times$ & $\checkmark$ & $\checkmark$ & $\checkmark$ \\
\textsf{LiNGAM} & - & - & - & - & - & - & $\times$ & $\checkmark$ & $\checkmark$ \\
\textsf{ANM} & - & - & - & - & - & - & $\checkmark$ & $\checkmark$ & $\checkmark$\\ \hline
\textsf{RCD} & - & - & - & - & - & - & \text{conf} & \text{conf} & \text{conf}\\
\textsf{CAMUV} & - & - & - & - & - & - & $\checkmark$ & \text{conf} & \text{conf}\\
\textsf{BUPL} & - & - & - & - & - & - &\text{conf} & \text{conf} & \text{conf}\\
\textsf{COCA} & - & - & - & - & - & - & $\checkmark$ & $\checkmark$ & $\checkmark$ \\
\textbf{\textsf{CLOUD}} &  $\pmb{\checkmark}$ \newline \resizebox{1.2cm}{!}{$(\Delta = 0.16)$} & $\pmb{\checkmark}$ \newline \resizebox{1.2cm}{!}{$(\Delta = 1.5)$} & \textbf{conf} \newline \resizebox{1.2cm}{!}{$(\Delta = 0.01)$} & $\pmb{\checkmark}$ \newline \resizebox{1.2cm}{!}{$(\Delta = 0.03)$} & $\pmb{\checkmark}$ \newline \resizebox{1.2cm}{!}{$(\Delta = 0.03)$} & $\pmb{\checkmark}$ \newline \resizebox{1.2cm}{!}{$(\Delta = 0.09)$} & $\pmb{\checkmark}$ \newline \resizebox{1.2cm}{!}{$(\Delta = 0.10)$} & $\pmb{\checkmark}$ \newline \resizebox{1.2cm}{!}{$(\Delta = 0.21)$} & $\ \pmb{\checkmark}$ \newline \resizebox{1.2cm}{!}{$(\Delta = 0.12)$} \\
\hline
\end{tabular}
\vspace{.5ex}
{\raggedright $\checkmark$ indicates that a method inferred the true causal direction. $\times$ indicates that the output of a method was wrong direction. $\thickapprox$ indicates that a method drew undisicive conclusion. \textbf{conf} indicates that a method inferred $M_{X \gets C \to Y}$. \par}
\end{table}

\subsubsection{Confounded case: SOS DNA Repair Network Dataset}\label{subsubsecC3b}
Finally, we tested \textsf{CLOUD}'s ability of detecting latent confounding variables in complex causal relationships using SOS DNA repair network in E.coli~\cite{ronen2002assigning}. This dataset describes the causal relationships at the protein level between genes, consisting of measurements for eight different genes under four distinct ultraviolet radiation conditions, with a total sample size of $n=200$. A ground truth network, as established by \cite{perrin2003gene}, is depicted in Fig. \ref{fig:SOS_DNA_repair_network}. 
The gene lexA has causal influence on all other genes, creating a situation where at least lexA is an unobserved common cause among variables downstream (children) of lexA. Therefore, we randomly selected pairs of child nodes of lexA and conducted experiments to detect the presence of an unobserved common cause (lexA) between each pair. We also note that while the experimental setup correctly represent confounded case, the correct directed cases could not be extracted, such as the arrow from lexA to umuDC. This is because there might be common causes between the two.

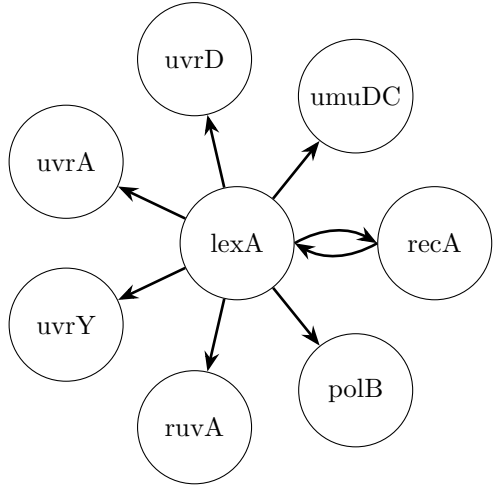
\begin{figure}[h]
    \centering
        \begin{tikzpicture}[>=Stealth, node distance=2cm]
        \node (lexA) at (0.00, 0.00) [circle, draw, minimum size=1.5cm] {lexA};
        \node (recA) at (2.60, 0.00) [circle, draw, minimum size=1.5cm] {recA};
        \node (umuDC) at (1.56, 1.95) [circle, draw, minimum size=1.5cm] {umuDC};
        \node (uvrD) at (-0.56, 2.44) [circle, draw, minimum size=1.5cm] {uvrD};
        \node (uvrA) at (-2.25, 1.08) [circle, draw, minimum size=1.5cm] {uvrA};
        \node (uvrY) at (-2.25, -1.08) [circle, draw, minimum size=1.5cm] {uvrY};
        \node (ruvA) at (-0.56, -2.44) [circle, draw, minimum size=1.5cm] {ruvA};
        \node (polB) at (1.56, -1.95) [circle, draw, minimum size=1.5cm] {polB};
        \draw [->, line width=1pt] (lexA) -- (umuDC);
        \draw [->, line width=1pt] (lexA) -- (uvrD);
        \draw [->, line width=1pt] (lexA) -- (uvrA);
        \draw [->, line width=1pt] (lexA) -- (uvrY);
        \draw [->, line width=1pt] (lexA) -- (ruvA);
        \draw [->, line width=1pt] (lexA) -- (polB);
        \draw [->, line width=1pt] (lexA.east) to[bend left] (recA.west);
        \draw [->, line width=1pt] (recA.west) to[bend left] (lexA.east);
        \end{tikzpicture}
        \caption{Ground truth graph of SOS DNA repair network}
        \label{fig:SOS_DNA_repair_network}
\end{figure}

\begin{figure}[h]
 \centering \includegraphics[width=1.01\linewidth]{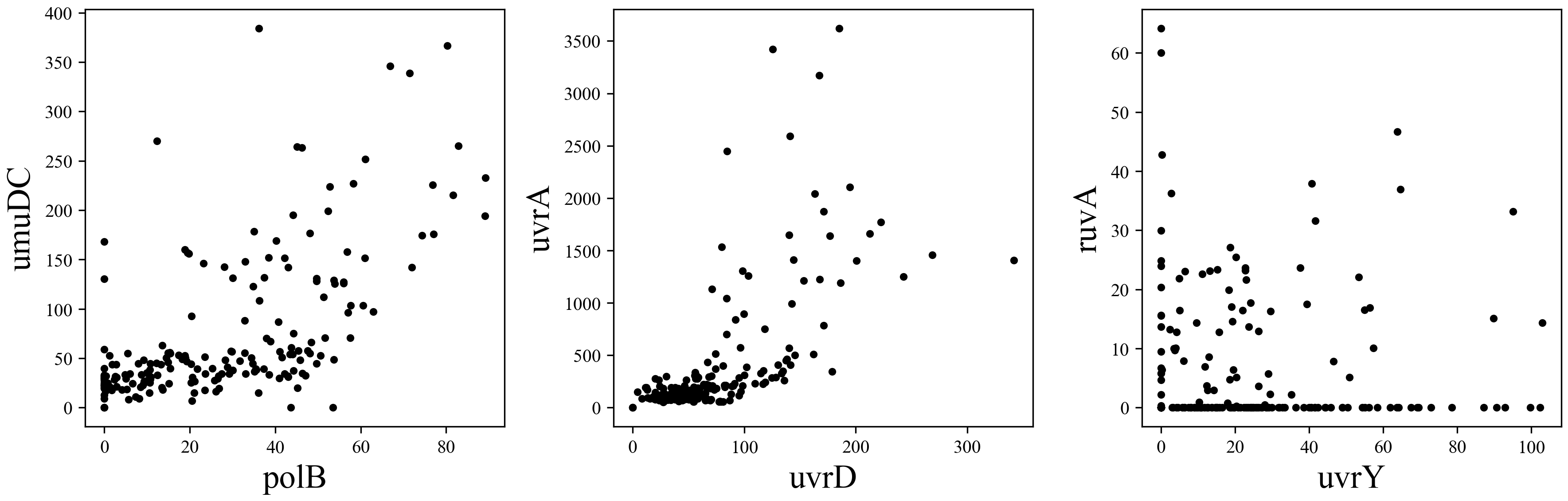}
 \caption{Scatter plots of variable pairs from the SOS DNA repair network dataset
 \label{fig:scatterSOSNetwork}}
\end{figure}

For comparison, we employed \textsf{RCD}, \textsf{CAMUV}, \textsf{BUPL}, and \textsf{COCA}, which are capable of inferring $M_{X \gets C \to Y}$. Since \textsf{COCA} considers asymmetry between $X$ and $Y$, we conducted experiments in two settings: one with $\Mcal = \lbrace M_{X \to Y}, M_{X \gets C \to Y} \rbrace$ (referred to as \textsf{COCA}($X \to Y$)) and another with $\Mcal = \lbrace M_{X \gets Y}, M_{X \gets C \to Y} \rbrace$ (referred to as \textsf{COCA}($X \gets Y$)).

The results are shown in Table \ref{tab:resultSOS}. For the pairs (polB, umuDC) and (uvrD, uvrA), the methods \textsf{CLOUD}, \textsf{RCD}, and \textsf{CAMUV} successfully identified the presence of unobserved common causes. However, for the pair (uvrY, ruvA), only \textsf{COCA} detected the unobserved common cause. We observed that \textsf{CLOUD} inferred causal independence for this pair, likely due to ruvA exhibiting zero-inflation and the statistical independence of the pair.

\textsf{CLOUD} demonstrated performance comparable to existing state-of-the-art methods that allow for unobserved common causes.

Given the prevalence of unobserved confounders in real-world applications, which can lead to incorrect causal conclusions if not properly addressed, we can conclude that \textsf{CLOUD} is equipped with a crucial feature for handling such scenarios. Moreover, considering the results from the Tübingen Benchmark Pairs as well, as presented in Table \ref{tab:resultTugingen}, we further affirm that \textsf{CLOUD} is a reliable method for real-world data analysis.

\begin{table}[h]
\centering
\caption{Results on SOS DNA repair network}
\label{tab:resultSOS}
\let\cline\clineorig
\renewcommand{\arraystretch}{1.4}
\addtolength{\tabcolsep}{-1pt}
\begin{tabular}{l>{\centering\arraybackslash}p{2cm}cccccc}
\toprule
Ground Truth    & \textbf{\textsf{CLOUD}}               &  \textsf{RCD}                &  \textsf{CAMUV}            &  \textsf{BUPL}     &  \textsf{COCA($X \to Y$)}      & \textsf{COCA($X \gets Y$)}\\ \midrule
polB $\gets C \to $ umuDC & \pmb{\checkmark} \par \resizebox{1.2cm}{!}{$(\Delta = 0.21)$} & $\checkmark$ & $\checkmark$ & $\checkmark$ & $\times$ & $\times$ \\ \hline
uvrD $\gets C \to $ uvrA  & \pmb{\checkmark} \par \resizebox{1.2cm}{!}{$(\Delta = 0.013)$}  & $\checkmark$ & $\checkmark$ & $\checkmark$ & $\times$ & $\times$ \\ \hline
uvrY $\gets C \to $  ruvA  & $\times$ \par \resizebox{1.2cm}{!}{$(\Delta = 0.25)$}   & $\times$              & $\times$            & $\times$ & $\checkmark$ & $\checkmark$ \\ \bottomrule
\end{tabular}
\vspace{.5ex}
{$\checkmark$ indicates that a method inferred $M_{X \gets C \to Y}$, while $\times$ indicates that a method did not.\par}
\end{table}

\section{Conclusion}\label{sec7}
This paper proposed \textsf{CLOUD}, a novel causal discovery method for causal relationships between two variables with unobserved common causes across discrete, mixed, and continuous data types. 

Based on the Reichenbach's common cause principle, we defined the Reichenbach problem as a problem to statistically infer the causal relationships among four models: one independent model, one with unobserved common cause, and two models with direct causality. 

By employing the NML code, \textsf{CLOUD} offers a model selection approach to solve the Reichenbach problem without relying on assumptions about unobserved common causes.
\textsf{CLOUD} formulates four models for each causal relationship and data-type. In particular, \textsf{CLOUD} expresses all joint distributions of $X$ and $Y$ under $M_{X \gets C \to Y}$, which enables us to avoid explicitly modeling $C$. \textsf{CLOUD} calculates the NML-based codelength of the observational data under those four models and then infers a causal relationship by selecting the corresponding causal model that achieves the shortest codelength. 
We successfully extended the \textsf{CLOUD} from discrete to continuous data, through discretization. \textsf{CLOUD} has a consistency with respect to selected models and it is theoretically proven.
 
Through both synthetic and real-world data experiments, \textsf{CLOUD} has proven its effectiveness in solving the Reichenbach problem with high accuracy and consistency. It stands out in its ability to identify causal directions with greater precision than existing methods, across a variety of data types and under complex data-generating conditions. Additionally, \textsf{CLOUD} has demonstrated a strong performance in detecting latent variables, showcasing its robustness and reliability in causal discovery. 

However, challenges remain, as evidenced by its performance on zero-inflated data in our final experiment where \textsf{CLOUD} mistakenly determined they are causally independent.
This implies the applicable range of \textsf{CLOUD} is still restricted, despite its ability to detect unobserved common causes without additional assumptions on it. 
Moreover, Additive Noise Model which we assume has known to be vulnerable in handling data with heteroscedastic noise, i.e noise variances are dependent of observed variables unlike \textsf{ANM}'s assumption. Recent research is actively addressing these challenges~\cite{tagasovska2020distinguishing, xu2022inferring, choi2023model}. 

Future work aims to broaden \textsf{CLOUD}'s scope, addressing its current SCM assumptions to enhance robustness and applicability across diverse data scenarios.

\bibliography{sn-article}

\begin{thebibliography}{39}
\providecommand{\natexlab}[1]{#1}
\providecommand{\url}[1]{{#1}}
\providecommand{\urlprefix}{URL }
\providecommand{\doi}[1]{\url{https://doi.org/#1}}
\providecommand{\eprint}[2][]{\url{#2}}
 \bibcommenthead

\bibitem[{Akaike(1998)}]{akaike1998information}
Akaike H (1998) Information theory and an extension of the maximum likelihood
  principle. In: Selected papers of hirotugu akaike. Springer, p 199--213

\bibitem[{Budhathoki and Vreeken(2017)}]{budhathoki2017mdl}
Budhathoki K, Vreeken J (2017) {MDL} for causal inference on discrete data. In:
  2017 IEEE International Conference on Data Mining (ICDM), pp 751--756

\bibitem[{Budhathoki and Vreeken(2018)}]{budhathoki2018accurate}
Budhathoki K, Vreeken J (2018) Accurate causal inference on discrete data. In:
  2018 IEEE International Conference on Data Mining (ICDM), pp 881--886

\bibitem[{Choi and Ni(2023)}]{choi2023model}
Choi J, Ni Y (2023) Model-based causal discovery for zero-inflated count data.
  Journal of Machine Learning Research 24(200):1--32

\bibitem[{Hirai and Yamanishi(2013)}]{hirai2013efficient}
Hirai S, Yamanishi K (2013) Efficient computation of normalized maximum
  likelihood codes for gaussian mixture models with its applications to
  clustering. IEEE Transactions on Information Theory 59(11):7718--7727

\bibitem[{Hoyer et~al(2008{\natexlab{a}})Hoyer, Janzing, Mooij, Peters, and
  Sch{\"o}lkopf}]{hoyer2008nonlinear}
Hoyer P, Janzing D, Mooij JM, et~al (2008{\natexlab{a}}) Nonlinear causal
  discovery with additive noise models. Advances in neural information
  processing systems 21

\bibitem[{Hoyer et~al(2008{\natexlab{b}})Hoyer, Shimizu, Kerminen, and
  Palviainen}]{hoyer2008estimation}
Hoyer PO, Shimizu S, Kerminen AJ, et~al (2008{\natexlab{b}}) Estimation of
  causal effects using linear non-gaussian causal models with hidden variables.
  International Journal of Approximate Reasoning 49(2):362--378

\bibitem[{Janzing and Sch\"{o}lkopf(2010)}]{janzing2010causal}
Janzing D, Sch\"{o}lkopf B (2010) Causal inference using the algorithmic
  {Markov} condition. IEEE Transactions on Information Theory 56(10):5168--5194

\bibitem[{Kaltenpoth and Vreeken(2019)}]{kaltenpoth2019we}
Kaltenpoth D, Vreeken J (2019) We are not your real parents: Telling causal
  from confounded using mdl. In: Proceedings of the 2019 SIAM International
  Conference on Data Mining, SIAM, pp 199--207,
  \urlprefix\url{https://github.com/davidkwca/CoCa}

\bibitem[{Kobayashi et~al(2022)Kobayashi, Miyaguchi, and
  Matsushima}]{kobayashi2022detection}
Kobayashi M, Miyaguchi K, Matsushima S (2022) Detection of unobserved common
  cause in discrete data based on the mdl principle. In: 2022 IEEE
  International Conference on Big Data (Big Data), IEEE, pp 45--54

\bibitem[{Kocaoglu et~al(2017)Kocaoglu, Dimakis, Vishwanath, and
  Hassibi}]{kocaoglu2017entropic}
Kocaoglu M, Dimakis AG, Vishwanath S, et~al (2017) Entropic causal inference.
  In: Thirty-First AAAI Conference on Artificial Intelligence

\bibitem[{Kontkanen and Myllym{\"a}ki(2007)}]{kontkanen2007linear}
Kontkanen P, Myllym{\"a}ki P (2007) A linear-time algorithm for computing the
  multinomial stochastic complexity. Information Processing Letters
  103(6):227--233

\bibitem[{Kontkanen et~al(2008)Kontkanen, Wettig, and
  Myllym{\"a}ki}]{kontkanen2008nml}
Kontkanen P, Wettig H, Myllym{\"a}ki P (2008) {NML} computation algorithms for
  tree-structured multinomial {Bayesian} networks. EURASIP Journal on
  Bioinformatics and Systems Biology 2007:1--11

\bibitem[{Li et~al(2022)Li, Xia, Liu, and Sun}]{li2022hybrid}
Li Y, Xia R, Liu C, et~al (2022) A hybrid causal structure learning algorithm
  for mixed-type data. In: Proceedings of the AAAI Conference on Artificial
  Intelligence, pp 7435--7443,
  \urlprefix\url{https://github.com/DAMO-DI-ML/AAAI2022-HCM}

\bibitem[{Liu and Chan(2016)}]{liu2016causal}
Liu F, Chan L (2016) Causal inference on discrete data via estimating distance
  correlations. Neural computation 28(5):801--814

\bibitem[{Maeda and Shimizu(2020)}]{maeda2020rcd}
Maeda TN, Shimizu S (2020) Rcd: Repetitive causal discovery of linear
  non-gaussian acyclic models with latent confounders. In: International
  Conference on Artificial Intelligence and Statistics, PMLR, pp 735--745

\bibitem[{Maeda and Shimizu(2021)}]{maeda2021causal}
Maeda TN, Shimizu S (2021) Causal additive models with unobserved variables.
  In: Uncertainty in Artificial Intelligence, PMLR, pp 97--106

\bibitem[{Marx and Vreeken(2021)}]{marx2021formally}
Marx A, Vreeken J (2021) Formally justifying mdl-based inference of cause and
  effect. arXiv preprint arXiv:210501902

\bibitem[{Mooij et~al(2016)Mooij, Peters, Janzing, Zscheischler, and
  Sch{\"o}lkopf}]{mooij2016distinguishing}
Mooij JM, Peters J, Janzing D, et~al (2016) Distinguishing cause from effect
  using observational data: methods and benchmarks. The Journal of Machine
  Learning Research 17(1):1103--1204

\bibitem[{Pearl(2009)}]{pearl2009causality}
Pearl J (2009) Causality. Cambridge university press

\bibitem[{Perrin et~al(2003)Perrin, Ralaivola, Mazurie, Bottani, Mallet, and
  d'Alch{\'e} Buc}]{perrin2003gene}
Perrin BE, Ralaivola L, Mazurie A, et~al (2003) Gene networks inference using
  dynamic bayesian networks. Bioinformatics-Oxford 19(2):138--148

\bibitem[{Peters et~al(2010)Peters, Janzing, and
  Sch{\"o}lkopf}]{peters2010identifying}
Peters J, Janzing D, Sch{\"o}lkopf B (2010) Identifying cause and effect on
  discrete data using additive noise models. In: Proceedings of the thirteenth
  international conference on artificial intelligence and statistics, JMLR
  Workshop and Conference Proceedings, pp 597--604

\bibitem[{Peters et~al(2011)Peters, Janzing, and Scholkopf}]{peters2011causal}
Peters J, Janzing D, Scholkopf B (2011) Causal inference on discrete data using
  additive noise models. IEEE Transactions on Pattern Analysis and Machine
  Intelligence 33(12):2436--2450

\bibitem[{Peters et~al(2014)Peters, Mooij, Janzing, and
  Sch{\"o}lkopf}]{peters2014causal}
Peters J, Mooij JM, Janzing D, et~al (2014) Causal discovery with continuous
  additive noise models. Journal of Machine Learning Research 15(58):2009--2053

\bibitem[{Peters et~al(2017)Peters, Janzing, and
  Sch{\"o}lkopf}]{peters2017elements}
Peters J, Janzing D, Sch{\"o}lkopf B (2017) Elements of causal inference:
  foundations and learning algorithms. The MIT Press

\bibitem[{Rissanen(1978)}]{rissanen1978modeling}
Rissanen J (1978) Modeling by shortest data description. Automatica
  14(5):465--471

\bibitem[{Rissanen(1983)}]{rissanen1983universal}
Rissanen J (1983) A universal prior for integers and estimation by minimum
  description length. The Annals of statistics 11(2):416--431

\bibitem[{Rissanen(1989)}]{rissanen1989stochastic}
Rissanen J (1989) Stochastic complexity in statistical inquiry. World
  Scientific

\bibitem[{Rissanen(2012)}]{rissanen2012optimal}
Rissanen J (2012) Optimal Parameter Estimation. Cambridge University Press

\bibitem[{Ronen et~al(2002)Ronen, Rosenberg, Shraiman, and
  Alon}]{ronen2002assigning}
Ronen M, Rosenberg R, Shraiman BI, et~al (2002) Assigning numbers to the
  arrows: parameterizing a gene regulation network by using accurate expression
  kinetics. Proceedings of the national academy of sciences 99(16):10555--10560

\bibitem[{Sch{\"o}lkopf(2022)}]{scholkopf2022causality}
Sch{\"o}lkopf B (2022) Causality for machine learning. In: Probabilistic and
  Causal Inference: The Works of Judea Pearl. p 765--804

\bibitem[{Shimizu et~al(2006)Shimizu, Hoyer, Hyv{\"a}rinen, Kerminen, and
  Jordan}]{shimizu2006linear}
Shimizu S, Hoyer PO, Hyv{\"a}rinen A, et~al (2006) A linear non-{Gaussian}
  acyclic model for causal discovery. Journal of Machine Learning Research
  7(10)

\bibitem[{Shimizu et~al(2011)Shimizu, Inazumi, Sogawa, Hyv{\"a}rinen, Kawahara,
  Washio, Hoyer, and Bollen}]{shimizu2011directlingam}
Shimizu S, Inazumi T, Sogawa Y, et~al (2011) Directlingam: A direct method for
  learning a linear non-gaussian structural equation model. The Journal of
  Machine Learning Research 12:1225--1248

\bibitem[{Spirtes et~al(2000)Spirtes, Glymour, and
  Scheines}]{spirtes2000causation}
Spirtes P, Glymour CN, Scheines R (2000) Causation, prediction, and search. MIT
  press

\bibitem[{Stegle et~al(2010)Stegle, Janzing, Zhang, Mooij, and
  Sch\"{o}lkopf}]{stegle2010probabilistic}
Stegle O, Janzing D, Zhang K, et~al (2010) Probabilistic latent variable models
  for distinguishing between cause and effect. Advances in neural information
  processing systems 23

\bibitem[{Tagasovska et~al(2020)Tagasovska, Chavez-Demoulin, and
  Vatter}]{tagasovska2020distinguishing}
Tagasovska N, Chavez-Demoulin V, Vatter T (2020) Distinguishing cause from
  effect using quantiles: Bivariate quantile causal discovery. In:
  International Conference on Machine Learning, PMLR, pp 9311--9323

\bibitem[{Tashiro et~al(2014)Tashiro, Shimizu, Hyv{\"a}rinen, and
  Washio}]{tashiro2014parcelingam}
Tashiro T, Shimizu S, Hyv{\"a}rinen A, et~al (2014) Parcelingam: A causal
  ordering method robust against latent confounders. Neural computation
  26(1):57--83

\bibitem[{Xu et~al(2022)Xu, Mian, Marx, and Vreeken}]{xu2022inferring}
Xu S, Mian OA, Marx A, et~al (2022) Inferring cause and effect in the presence
  of heteroscedastic noise. In: International Conference on Machine Learning,
  PMLR, pp 24615--24630

\bibitem[{Zeng et~al(2022)Zeng, Shimizu, Matsui, and Sun}]{zeng2022causal}
Zeng Y, Shimizu S, Matsui H, et~al (2022) Causal discovery for linear mixed
  data. In: Conference on Causal Learning and Reasoning, PMLR, pp 994--1009

\end{thebibliography}
\bibliographystyle{sn-basic}

\begin{appendices}
\section{Optimization of $f$ and $g$ through Likelihood Maximization}\label{secA1}
The algorithm for estimating the function $\hat{f}$ is shown in the 
Algorithm~\ref{alg:optimize_regression}.
In order to compute both terms in Eq.~\eqref{eq:disc_x2y_two_part_code}, we must estimate $\hat{f}\colon\Xcal\to \Ycal$ that achieves a higher likelihood to calculate a shorter codelength of data under $M_{X \to Y}$.  
First, we initialize the function $\fhat$ that returns the most frequent $y$ for each $x\in\Xcal$ (lines 1-3). Subsequently, we iteratively update the function value for $x \in \mathcal{X}$. At $j$-th step, we update the function value $f(x)$ while fixing all other mapping $f(x')$ for $x' \neq x$ and check if the likelihood increases. If it increases, we change the value $f(x)$ to new $y$. 
This update is done for all $x \in \Xcal$ and repeated until the likelihood no longer increases or for at most $J$ times. 
Similarly, the function $g: \Ycal \to \Xcal$ can be estimated by replacing $X$ and $Y$ in algorithm~\ref{alg:optimize_regression}.

\begin{algorithm}[t]                      
\caption{Optimize Regression Function $f: \Xcal \to \Ycal$ with Likelihood Maximization}
\label{alg:optimize_regression}                          
\begin{algorithmic}[1]
\Require $ z^n \in \mathcal{X}^n \times \mathcal{Y}^n\,$: Dataset with sample size of $n$ \\  $\quad\ J$ : Maximum number of update iterations  
\Ensure $\hat{f}\colon \Xcal \to \Ycal$ : Estimated Regression Function 
\For {$x \in \mathcal{X}$}
    \State $f^{(0)}(x) \gets \argmax\limits_{y \in \mathcal{Y}} n(X=x, Y=y)$
\EndFor
\State $j \gets 0$
\State $r_{\mathrm{max}} \gets \max\limits_{\thetab \in \bm{\Theta}_{X\to Y} } P\!\left(z^n; M_{X \to Y} , f^{(0)}, \thetab \right)$
\While{$\mathit{converged}=\mathrm{False}$ \textbf{or} $j < J$}
    \State $j \gets j + 1$
    \State $\mathit{converged} \gets \mathrm{True}$
    \For {$x \in \mathcal{X}$}
        \State $r \gets \max\limits_{f^{(j-1)}(x)\in \Ycal} \max\limits_{\thetab \in \bm{\Theta}_{X\to Y} } P\!\left(z^n; M_{X \to Y}, f^{(j-1)}, {\thetab} \right)$ 
        \Comment{subject to $f^{(j-1)}$ is not constant}
        \If{$r > r_{\mathrm{max}}$}
            \State $r_{\mathrm{max}} \gets r$
            \State $f^{(j)}(x) \gets  \argmax\limits_{f^{(j-1)}(x)\in \Ycal} \max\limits_{\thetab \in \bm{\Theta}_{X\to Y} } P\!\left(z^n; M_{X \to Y}, f^{(j-1)}, \bm{\theta}\right)$ 
            \Comment{subject to $f^{(j-1)}$ is not constant}
            \State $\mathit{converged} \gets \mathrm{False}$
        \EndIf
    \EndFor
\EndWhile
\State \textbf{return $f^{(j)}$}
\end{algorithmic}
\end{algorithm}

\section{Proof of Proposition 1}\label{secB1}
In this section, we provide the proof for Proposition 1.
\subsection{Independent Case}\label{subsecB1}
We exactly calculate the stochastic complexity defined in Eq.~\eqref{stochastic_complexity} for causal models $M_{X \Perp Y} $:
\begin{align*}
 &\Lcal^{d}(z^n; M_{X \Perp Y}) \\
 &=\mathcal{SC}(z^n; M_{X \Perp Y}) \\
 &= - \log P(z^n; M_{X \Perp Y}, {\hat{\bm{\theta}}}_{X \Perp Y}\left(z^n\right))
  + \log  \mathcal{C}_{n}(M_{X \Perp Y})
\end{align*}

Using the result of likelihood estimation $P(X=k,Y=k'; M_{X \Perp Y}, {\hat{\bm{\theta}}}_{X \Perp Y}\left(z^n\right)) =  \frac{n(X = k)}{n}\frac{n(Y = k')}{n}$, the maximum likelihood which the first term on the right-hand side is represented as:
\begin{align*}
  \log P(z^n; M_{X \Perp Y}, {\hat{\bm{\theta}}}_{X \Perp Y}) =\sum_{k=0}^{m_X - 1} n(X = k) \log \frac{n(X = k)}{n}
    +\sum_{k'=0}^{m_Y - 1} n(Y = k') \log \frac{n(Y = k')}{n},
\end{align*}
where ${\hat{\bm{\theta}}}_{X \Perp Y}=(\hat{\theta}_{k, k'} )$ is the maximum likelihood estimator in $\bm{\Theta}_{{X \Perp Y}}$.
The parametric complexity for $M_{X \Perp Y}$ as the second term is calculated as
\begin{align*}
    & \log\mathcal{C}_{n}(M_{X \Perp Y}) \\
    &=\log \sum_{Z^n\in \Xcal^n \times \Ycal^n} \max_{\thetab \in \bm{\Theta}_{X \Perp Y}} P(Z^n;\, M_{X\Perp Y}, \thetab) \\
    &= \log \left\lbrace \sum_{Z^n} \prod_{k=0}^{m_X -1} \left(\frac{n(X=k)}{n}\right)^{n(X=k)} \prod_{k'=0}^{m_Y -1} \left(\frac{n(Y=k')}{n}\right)^{n(Y=k')} \right \rbrace\\
    &= \log \left\lbrace \sum_{X^n} \prod_{k=0}^{m_X -1} \left(\frac{n(X=k)}{n}\right)^{n(X=k)} \sum_{Y^n} \prod_{k'=0}^{m_Y -1} \left(\frac{n(Y=k')}{n}\right)^{n(Y=k')} \right \rbrace\\
    &= \log \mathcal{C}_{\mathsf{CAT}}(K=m_X, n) + \log \mathcal{C}_{\mathsf{CAT}}(K=m_Y, n).
\end{align*}
Therefore, the NML codelength of the data $z^n$ for $M_{X \Perp Y}$ is calculated as:
\begin{align*}
    \Lcal^{d}(z^n ; M_{X \Perp Y}) 
    &= \sum_{k=0}^{m_X - 1} n(X = k) \log \frac{n(X = k)}{n} + \log \mathcal{C}_{\mathsf{CAT}}(K=m_X, n) \\
    &\phantom{=}+\sum_{k'=0}^{m_Y - 1} n(Y = k') \log \frac{n(Y = k')}{n} + \log \mathcal{C}_{\mathsf{CAT}}(K=m_Y, n),
\end{align*}
which is equal to $\mathcal{SC}(z^n; \mathsf{CAT}^{m_X})$.

\subsection{Confounded Case}\label{subsecB2}
We exactly calculate the stochastic complexity defined in Eq.~\eqref{stochastic_complexity} for causal model $M_{X \gets C \to Y}$:
\begin{align*}
   &\Lcal^{d}(z^n; M_{X \gets C \to Y})\\
   &= \mathcal{SC}(z^n; M_{X \gets C \to Y})\\
  &= - \log P(z^n; M_{X \gets C \to Y}, {\hat{\bm{\theta}}}_{X \gets C \to Y}) \notag +\log  \mathcal{C}_{n}(M_{X \gets C \to Y}),
\end{align*}
where
${\hat{\bm{\theta}}}_{X \gets C \to Y}=(\hat{\theta}_{k, k'} )$ is the maximum likelihood estimator in
$\bm{\Theta}_{{X \gets C \to Y}}$.
Each element is $\hat{\theta}_{k, k'} = \frac{n\left(X=k, Y=k'\right)}{n}$, where
$n\left(X=k, Y=k'\right)$ counts the frequency of data satisfying $X=k$ and $Y=k'$ in $z^n$. Subsequently, the maximum likelihood as the first term on the right-hand side is represented as
\begin{align*}
   \log P(z^n; M_{X \gets C \to Y}, {\hat{\bm{\theta}}}_{X \gets C \to Y})  
    =\sum_{k=0}^{m_X - 1} \sum_{k'=0}^{m_Y - 1} n(X = k, Y=k') \log \frac{n(X = k, Y = k')}{n}.
\end{align*}
Since the causal model of $M_{X\gets C\to Y}$ is the model of $(m_Xm_Y)$-categorical distributions, 
the parametric complexity as the second term is calculated as 
\begin{align*}
     \log\mathcal{C}_{n}(M_{X \gets C \to Y}) =  \log\mathcal{C}_{\mathsf{CAT}}(m_Xm_Y, n).
\end{align*}

Thus, the NML codelength of the data $z^n$ for $M_{X \gets C \to Y}$ is calculated as:
\begin{align*}
&\Lcal^{d}(z^n; M_{X \gets C \to Y}) \\
&=-\sum_{k=0}^{m_X - 1} \sum_{k'=0}^{m_Y - 1} n\left(X = k, Y=k'\right) \log \frac{n\left(X = k, Y = k'\right)}{n} + \log  \mathcal{C}_{\mathsf{CAT}}(m_Xm_Y, n), 
\end{align*}
which is equal to $\mathcal{SC}(z^n; \mathsf{CAT}^{m_X m_Y}) = \mathcal{SC}(z^n; \mathsf{HIS}^{m_X, m_Y})$.

\subsection{Direct Case}\label{subsecB3}
First, we consider the first term on the right-hand side of Eq. \eqref{eq:disc_x2y_two_part_code}.
The first term, $L(f; M_{X \to Y})$, represents the codelength required to select one function from a finite set of possible functions $f \in \Fcal$, and the following holds true:
\begin{thm}
\label{th:function_code_length}
Let $\Fcal$ be a set of non-constant functions from $\Xcal$ to $\Ycal$. Then, we see 
\begin{align}
\label{eq:func_code}
    L(f; M_{X \to Y}) =  \log |\mathcal{F}| = \log (m_Y^{m_X - 1} - 1)
\end{align}
\end{thm}
\begin{proof} 
Naively, the number of the possible functions $f$ amounts to $m_Y^{m_X}$, but one can remove redundant functions to shorten the resulting codelength.

First, one can remove constant functions since they are associated with the independence model $M_{X\Perp Y}$ and not $M_{X\to Y}$. $m_Y$ such functions exist in total.

Next, distinct functions are associated with the same NML codelength.
For a function $f_1$, consider $f_2$ given by a constant shift,
\begin{align*}
&f_2(x) = f_1(x)+k' \pmod{m_Y},
\end{align*}
where $k' \in \lbrace 1, \cdots , m_Y - 1\rbrace$ and $x \in \mathcal{X}$.
Now, for all $k'$, the NML codelength with $f_1$ is the same as that with $f_2$.
Since $m_Y$ such different but equivalent functions exist for any $f_1$ including itself, one can further reduce the number of functions by a factor of $m_Y$.

Summing up, there remain $|\Fcal|=\left(m_Y^{m_X} - m_Y\right) / m_Y = m_Y^{m_X - 1} - 1$ functions to encode. 
\end{proof}

As for the second term in Eq. \eqref{eq:disc_x2y_two_part_code}, $L(z^n; M_{X\to Y}, f )$, the following statement holds:
\begin{thm}
\label{th:appro_nml_code_lenght}
We define 
$n(Y=f(X) + k')$ as the frequency of data in $z^n$ that satisfies $Y=f(X) + k'$.
Then, for any $f$, it holds that
\begin{align*}
&L(z^n ; M_{X\to Y}, f )  \\
  &=- \sum_{k=0}^{m_X - 1} n(X=k) \log \frac{n(X=k)}{n} 
 - \sum_{k'=0}^{m_Y - 1} n(Y=f(X) + k') \log \frac{n(Y=f(X) + k')}{n} \\
  &+ \log  \mathcal{C}_{\mathsf{CAT}}(K=m_X, n) + \log \mathcal{C}_{\mathsf{CAT}}(K=m_Y, n) .
\end{align*}
\end{thm}
\begin{proof} 
Let us denote the probability mass functions of $E_X$ and $E_Y$ by
$P(E_X; \bm{\pi}_{X})$ and $P(E_Y; \bm{\pi}_{Y})$, respectively, where $\bm{\pi}_{X}, \bm{\pi}_{Y}$ are the corresponding parameters.

Now, the observable pair $(X, Y)$ is one-to-one with the exogenous variable pair $(E_X, E_Y)$ when the function $f$ is fixed. 
Thus, under an appropriate transformation of data, the joint probability mass function of $(X, Y)$ is equivalent to that of $(E_X, E_Y)$,
\begin{align*}
    &P(X=x, Y=y; M_{X \to Y}, f, \bm{\pi}_{X}, \bm{\pi}_{Y}) \\
    &= P(E_X=x,E_Y=y-f(x); M_{X \to Y}, f, \bm{\pi}_{X}, \bm{\pi}_{Y}) \\
    &= P(E_X=x'; \bm{\pi}_{X}) P(E_Y=y'; \bm{\pi}_{Y}),
\end{align*}
which implies the equivalence of $M_{X \to Y}$ with fixed $f$ and $M_{X \Perp Y}$
\begin{align*}
    &P(X=x, Y=y; M_{X \to Y}, f, \bm{\pi}_{X}, \bm{\pi}_{Y}) \\
    &=P(X=x', Y=y'; M_{X \Perp Y}, \bm{\pi}_{X}, \bm{\pi}_{Y}),
\end{align*}
where $z'=(x', y')=(x, y-f(x))$ is the transformation of a datum $z=(x,y)$ with fixed $f$.

The equivalence in terms of the probability mass functions immediately extends to the equivalence in terms of the NML codelengths.
Particularly, the NML codelength of $M_{X\to Y}$ with fixed $f$ is the same as that of 
$M_{X\Perp Y}$ with the appropriate transformation,
\begin{align*}
    L(z^n ; M_{X\to Y}, f)  = \Lcal^{d}(z'^n; M_{X \Perp Y}),
\end{align*}
where $z'^n=(z'_1,...,z'_n)$ and $z'_i=(x_i, y_i-f(x_i))$ for all $1\le i \le n$.
This completes the proof.
\end{proof}

Therefore, the codelength of $z^n \in \mathcal{X}^n \times \mathcal{Y}^n = \lbrace 0, 1, \cdots m_X -1 \rbrace^n \times \lbrace 0, 1, \cdots m_Y - 1 \rbrace^n$ for the causal model $ M_{X \to Y}$ is calculated as:
\begin{align*}
& \Lcal^{d}(z^n ; M_{X \to Y}) \\ 
&= L(z^n ; M_{X\to Y}, \fhat ) + L(\fhat; M_{X\to Y})\\
&= - \sum_{k=0}^{m_X - 1} n(X=k) \log \frac{n(X=k)}{n}  
- \sum_{k'=0}^{m_Y - 1} n(Y=\fhat(X) + k') \log \frac{n(Y=\fhat(X) + k')}{n} \\
&\phantom{=}+ \log  \mathcal{C}_{\mathsf{CAT}}(K=m_X, n) + \log \mathcal{C}_{\mathsf{CAT}}(K=m_Y, n) + \log (m_Y^{m_X - 1} - 1),
\end{align*}
which is equal to $\mathcal{SC}(x^n; \mathsf{CAT}^{m_X}) + \mathcal{SC}((y - \fhat(x))^n; \mathsf{CAT}^{m_Y}) + L(\fhat; M_{X \to Y})$. The subtraction is taken over $\mathbb{Z}/m_Y\mathbb{Z}$.

\section{Proof of Proposition 2}\label{secC1}
In this section, we provide the proof for Proposition 2. 

We first derive the first term on the right-hand side of Eq. \eqref{eq:Lc_2stg_mXmY}, $L(m_X, m_Y; M)$, 
which is the codelength required to encode $(m_X, m_Y)$ under causal model $M$. By the Rissanen's integer coding (Eq. \eqref{eq:RissanenIntCode}) we have:
\begin{align}
    L(m_X, m_Y; M) = \log^{*}m_X + \log^{*}m_Y.
\end{align}

For the second term on the right-hand side of Eq. \eqref{eq:Lc_2stg_mXmY}, $L(z^n; M, m_X, m_Y)$, we calculate it under each causal model in the following sections, and we complete the proof of \eqref{eq:Lc_MmXmY} in Proposition \ref{prop2} for each causal relationship:

\subsection{Independent Case}\label{subsecC1}
 For given data $z^n=(x^n,y^n)$,  codelength $L(z^n; M_{X \Perp Y}, m_X, m_Y)$ is calculated by two NML codes with respect to the histogram models $\mathsf{HIS}^{m_X}$ and $\mathsf{HIS}^{m_Y}$ for $x^n$ and $y^n$, respectively. That is, we have:
\begin{align*}
    L(z^n; M_{X \Perp Y}, m_X, m_Y) = \mathcal{SC}(x^n; \mathsf{HIS}^{m_X}) +  \mathcal{SC}(y^n; \mathsf{HIS}^{m_X}).
\end{align*}
Then, the total codelength $\Lcal^{c}(z^n, m_X, m_Y; M_{X \Perp Y})$ is expressed as follows:
\begin{align*}
    \Lcal^{c}(z^n, m_X, m_Y; M_{X \Perp Y}) 
    &=L(m_X, m_Y; M_{X \Perp Y}) + L(z^n; M_{X \Perp Y}, m_X, m_Y) \\ 
    &= \log^{*}m_X + \log^{*}m_Y + \mathcal{SC}(x^n; \mathsf{HIS}^{m_X}) +  \mathcal{SC}(y^n; \mathsf{HIS}^{m_X}).
\end{align*}
By the results of Example \ref{ex:Example_Discrete} and \ref{ex:Example_Continuous}, the above is written as follows:
\begin{align*}
& \Lcal^{c}(z^n , m_X, m_Y; M_{X \Perp Y}) \\
&= \log^{*} m_X+ \log^{*} m_Y+ \mathcal{SC}(x^n; \mathsf{HIS}^{m_X}) + \mathcal{SC}(y^n; \mathsf{HIS}^{m_X})\\
&= -  \sum_{k=0}^{m_X-1} n(X \in I^X_k)\log \frac{n(X \in I^X_k)}{n} + \log \mathcal{C}_{\mathsf{CAT}}(K=m_X, n) - n \log m_X + \log^{*} m_X \\
&\phantom{=}-  \sum_{k'=0}^{m_Y-1} n(Y \in I^Y_{k'})\log \frac{n(Y \in I^Y_{k'})}{n} + \log \mathcal{C}_{\mathsf{CAT}}(K=m_Y, n) - n \log m_Y + \log^{*} m_Y \\
&=-  \sum_{k=0}^{m_X-1} n(\text{disc}(X;m_X) =k)\log \frac{n(\text{disc}(X;m_X) =k)}{n} + \log \mathcal{C}_{\mathsf{CAT}}(K=m_X, n) + L^{c \to d}(m_X, n)\\
&\phantom{=}-  \sum_{k'=0}^{m_Y-1} n(\text{disc}(Y;m_Y) =k')\log \frac{n(\text{disc}(Y;m_Y) =k')}{n} + \log \mathcal{C}_{\mathsf{CAT}}(K=m_Y, n) + L^{c \to d}(m_Y, n)\\
&= \Lcal^{d}(\text{disc}(x^n; m_X), \text{disc}(y^n; m_Y); \text{DISC}(M_{X \Perp Y}; m_X, m_Y)) + L^{c \to d}(m_X, n) + L^{c \to d}(m_Y, n).
\end{align*}
Thus, our claim holds in case $M=M_{X \Perp Y}$.

\subsection{Confounded Case}\label{subsecC2}
The codelength $L(z^n; M_{X \gets C \to Y}, m_X, m_Y)$ is calculated by NML code with respect to the histogram model $\mathsf{HIS}^{m_X, m_Y}$ for $z^n$. That is, we have: 
\begin{align*}
    L(z^n; M_{X \gets C \to Y}, m_X, m_Y) &= \mathcal{SC}(z^n; \mathsf{HIS}^{m_X, m_Y})\\
    &=-\max_{p\in \mathsf{HIS}^{m_X, m_Y}}\log p(z^n; \bm{\hat{\theta}}(z^n))+\log \mathcal{C}_n(\mathsf{HIS}^{m_X, m_Y}) 
\end{align*}
The maximum likelihood estimator for $\mathsf{HIS}^{m_X, m_Y}$ results in $\hat{\theta}_{k,k'}(z^n) = \frac{n(X \in I^{X}_k, Y \in I^{Y}_{k'})}{n}m_X m_Y$. Thus, the maximum log-likelihood of data is calculated as
\begin{align*}
    &\max_{p\in \mathsf{HIS}^{m_X, m_Y}} \log p(z^n; \bm{\hat{\theta}}(z^n))\\
    &= \sum_{k=0}^{m_X-1}\sum_{k'=0}^{m_Y-1} n(X \in I^{X}_k, Y \in I^{Y}_{k'}) \log \hat{\theta}_{k,k'}(z^n)\\
    &= \sum_{k=0}^{m_X-1} \sum_{k'=0}^{m_Y-1} n(X \in I^{X}_k, Y \in I^{Y}_{k'}) \log \frac{n(X \in I^{X}_k, Y \in I^{Y}_{k'})}{n} + n \log(m_X m_Y).
\end{align*}
The parametric complexity of $\mathsf{HIS}^{m_X, m_Y}$ is given by
\begin{align*}
     \log \mathcal{C}_n(\mathsf{HIS}^{m_X, m_Y}) 
    &= \log \int p (z^n; \bm{\hat{\theta}}(z^n)) dz^n \\
    &= \log \sum_{Z^n \in \lbrace 0, \ldots, m_X-1 \rbrace^n \times \lbrace 0, \ldots, m_Y-1 \rbrace^n} \int_{\Delta (Z^n)} p (z^n; \bm{\hat{\theta}}(z^n)) dz^n \\
    &= \log \sum_{Z^n \in \lbrace 0, \ldots, m_X-1 \rbrace^n \times \lbrace 0, \ldots, m_Y-1 \rbrace^n} \left( \frac{n(X \in I^{X}_k, Y \in I^{Y}_{k'})}{n} \right)^n \\
    &= \log \mathcal{C}_{\mathsf{CAT}}(K=m_X  m_Y, n) 
\end{align*}
Consequently, the NML codelength of $z^n$ for $\mathsf{HIS}^{m_X, m_Y}$ becomes
\begin{align}\label{SC_2dHIS}
    &\mathcal{SC}(z^n; \mathsf{HIS}^{m_X, m_Y}) \notag \\
    &= - \sum_{k=0}^{m_X-1} \sum_{k'=0}^{m_Y-1} n(X \in I^{X}_k, Y \in I^{Y}_{k'}) \log \frac{n(X \in I^{X}_k, Y \in I^{Y}_{k'})}{n} - n \log(m_X m_Y) \notag  \\
    &\phantom{=} +  \log \mathcal{C}_{\mathsf{CAT}}(K=m_X  m_Y, n).
\end{align}
By comparing the result with Example \ref{ex:Example_Discrete}, the total codelength $\Lcal^{c}(z^n, m_X, m_Y; M_{X \gets C \to Y})$ obtains the following representation:
\begin{align*}
& \Lcal^{c}(z^n , m_X, m_Y, M_{X \gets C \to Y}) \\
&= \log^{*} m_X+ \log^{*} m_Y+ \mathcal{SC}(z^n; \mathsf{HIS}^{m_X, m_Y}) \\
&= - \sum_{k=0}^{m_X - 1} \sum_{k'=0}^{m_Y - 1} n(X \in I^X_k, Y \in I^Y_{k'}) \log \frac{n(X \in I^X_k, Y \in  I^Y_{k'})}{n}  - n \log (m_X  m_Y) \\
&\phantom{=} + \log \mathcal{C}_{\mathsf{CAT}}(K=m_X  m_Y, n) + \log^{*}m_X + \log^{*}m_Y \\
&= - \sum_{k=0}^{m_X - 1} \sum_{k'=0}^{m_Y - 1} n(\text{disc}(X;m_X) =k,\text{disc}(Y;m_Y) =k') \log \frac{n(\text{disc}(X;m_X) =k,\text{disc}(Y;m_Y) =k')}{n} \\
&\phantom{=} + \log \mathcal{C}_{\mathsf{CAT}}(K=m_X  m_Y, n) + L^{c \to d}(m_X)+  L^{c \to d}(m_Y) \\
&= \Lcal^{d}(\text{disc}(x^n; m_X) , \text{disc}(y^n; m_Y); \text{DISC}(M^{m_X, m_Y}_{X \gets C \to Y})) + L^{c \to d}(m_X, n) + L^{c \to d}(m_Y, n).
\end{align*}
This completes the proof in case $M=M_{X \gets C \to Y}$.

\subsection{Direct Case}\label{subsecC3}
In this case, we also employ two-stage coding for $L(z^n; M_{X \to Y}, m_X, m_Y)$ with respect to function $f$. That is, for a given $(m_X, m_Y)$, we first estimate the optimal function $\hat{f}$ using maximum likelihood estimation (Algorithm \ref{alg:optimize_regression}), and then encode the data $z^n$ based on the histogram model $\mathsf{HIS}^{m_X, m_Y}_{X \to Y}$ with $\hat{f}$ fixed:
\begin{align*}
    L(z^n; M_{X \to Y}, m_X, m_Y) = L(\fhat; M_{X \to Y}, m_X, m_Y) + L(z^n; M_{X \to Y}, m_X, m_Y, \fhat),
\end{align*}
where the first term on the right-hand side is given by Eq. \eqref{eq:func_code} and the second one is calculated by encoding $x^n$ and $(y - \hat{f}(x))^n$ based on $\mathsf{HIS}^{m_X}$ and $\mathsf{HIS}^{m_Y}$, respectively.

Therefore, the codelength is formulated as follows:
\begin{align*}
    &\Lcal^{c}(z^n; m_X, m_Y; M_{X \to Y}) \\
    & = L(m_X, m_Y; M_{X \to Y}) + L(\fhat; M_{X \to Y}, m_X, m_Y) + L(z^n; M_{X \to Y}, m_X, m_Y, \fhat) \\
    &= \log^{*}m_X + \log^{*}m_Y + \log m_Y^{m_X - 1} + \mathcal{SC}(x^n; \mathsf{HIS}^{m_X}) +  \mathcal{SC}(y^n - \hat{f}(x^n); \mathsf{HIS}^{m_Y}),
\end{align*}

By the results of Example \ref{ex:Example_Discrete} and~Example~\ref{ex:Example_Continuous}, the total codelength in the above is expressed as follows:
\begin{align*}
& \Lcal^{c}(z^n ; m_X, m_Y; M_{X \to Y}) \\
&= \log^{*} m_X+ \log^{*} m_Y + \log m_Y^{m_X - 1} \\
&\phantom{=} - \sum_{k=0}^{m_X-1} n(X \in I^X_k)\log \frac{n(X \in I^X_k)}{n} + \log \mathcal{C}_{\mathsf{CAT}}(K=m_X, n) - n \log m_X \\
&\phantom{=} - \sum_{k'=0}^{m_Y-1} n(Y - f(X) \in I^Y_{k'})\log \frac{n(Y - f(X) \in I^Y_{k'})}{n} + \log \mathcal{C}_{\mathsf{CAT}}(K=m_Y, n) - n \log m_Y\\
&= - \sum_{k=0}^{m_X-1} n(\text{disc}(X;m_X) =k)\log \frac{n(\text{disc}(X;m_X) =k)}{n} + \log \mathcal{C}_{\mathsf{CAT}}(K=m_X, n)\\
&\phantom{=} - \sum_{k'=0}^{m_Y-1} n(\text{disc}(Y - \hat{f}(X);m_Y)=k')\log \frac{n(\text{disc}(Y - \hat{f}(X);m_Y)=k')}{n} + \log \mathcal{C}_{\mathsf{CAT}}(K=m_Y, n) \\
&\phantom{=}+\log m_Y^{m_X - 1}+ L^{c \to d}(m_X, n)+ L^{c \to d}(m_Y, n)\\
&= \Lcal^{d}(\text{disc}(x^n; m_X), \text{disc}(y^n; m_Y); \text{DISC}(M_{X \to Y}; m_X, m_Y)) + L^{c \to d}(m_X, n) + L^{c \to d}(m_Y, n).
\end{align*}

This completes the proof in case $M=M_{X \to Y}$.
The same argument holds for $M=M_{X \gets Y}$

\section{Description of T\"{u}bingen Benchmark Pairs}\label{secD1}
We provide the detailed descriptions of every dataset \cite{mooij2016distinguishing} we employed in Section \ref{subsubsecC3a}.

\subsubsection*{Discrete Case:}

\textbf{Traffic Dataset (No. 47)}:
This dataset focused on the relationship between the type of day and traffic volume. $X$ represents the number of cars counted per 24 hours at various stations in Oberschwaben, Germany. $Y$ is categorical, distinguishing between Sundays plus holidays (labelled as '1') and working days (labelled as '2'). The ground truth is $X \gets Y$, suggesting that the type of day influences the traffic volume.

\textbf{Internet Connections and Traffic Dataset (No. 68)}:
This dataset comes from a time series study focusing on internet connections and traffic at the MPI for Intelligent Systems. It features $X$ which represents the bytes sent at minute and $Y$ which denotes the number of open HTTP connections during that same minute. Measurements were taken every 20 minutes. The established ground truth is $Y \text{(open HTTP connections) }$ causes $X \text{(bytes sent)}$.

\textbf{Direction of Gabor Patches Dataset (No. 107)}:
This dataset originated from a psychophysics experiment involving human subjects and their perception of Gabor patches (stripe patterns used in psychological experiments) displayed on a screen. The Gabor patches were tilted either to the left or right, with varying contrast levels. $X$ represents the contrast values, ranging from 0.0150 to 0.0500, in increments of 0.0025, and $Y$ is binary, indicating whether the direction of the tilt was correctly identified or not. $X$ is regarded as the cause of $Y$.

\subsubsection*{Mixed Case:}
\textbf{Milk Protein Dataset (No. 85)}: The pair0085 dataset used in our experiments is a subset of the milk protein trial dataset by Verbyla and Cullis in 1990. The dataset contains weekly measurements of the assayed protein content of milk samples taken from 71 cows over a 14-week period. The cows were randomly allocated to one of three diets: barley, mixed barley-lupins, and lupins. The variables in the dataset are $X$, representing the time at which the weekly measurement was taken (ranging from 1 to 14), and $Y$, representing the protein content of the milk produced by each cow at time $X$. The ground truth for our experiments was set as $X \to Y$. Note that the dataset does not consider the effect of the diets on the protein content.

\textbf{Electricity Consumption Dataset (No. 95)}: This comprises 9,504 hourly measurements of total electricity consumption in MWh, denoted as $Y$, in a region of Turkey. The variable $X$ represents the hour of the day during which these measurements were taken. The ground truth for this dataset is set as $X \to Y$, suggesting that the hour of the day is the driving factor for electricity consumption.

\textbf{NLSchools Dataset (No. 99)}:
This dataset contains the information of 2287 Dutch eighth graders (about 11 years old) with features $X$(language test score) and $Y$(social-economic status of pupil's family). $X \to Y$ is regarded as the ground truth.

\subsubsection*{Continuous Case:}
\textbf{Cardiac Arrhythmia Database (No. 23)}:
The data, contributed in January 1998, includes 452 instances with two attributes: age and weight. Age was hypothesized to influence weight.

\textbf{Solar Radiation and Air Temperature Dataset (No. 77)}: This contains daily measurements of solar radiation in \textit{$W/m^2$} and the daily average temperature of the air in Furtwangen, Black Forest, Germany. The dataset covers a time period from January 1, 1985, to December 31, 2008, with a sample size of 8,401. Solar radiation is denoted by the variable $Y$, while the air temperature is denoted by the variable $X$. The ground truth for this dataset was set as $X \gets Y$, indicating that solar radiation is the cause of air temperature.

\textbf{Brightness of screen Dataset (No. 101)}:
This is from an experiment that was performed to generate samples that were clearly unconfounded. $X$ is grey value of a pixel randomly chosen from a fixed image. The grey value was displayed by the color of a square on a computer screen.
$Y$ is light intensity seen by a photo diode placed several centimeters away from the screen. $X$ was the cause of $Y$.

\end{appendices}


\end{document}